\documentclass{article}
\usepackage[nonatbib,final]{neurips_2020}
\usepackage[utf8]{inputenc} % allow utf-8 input
\usepackage[T1]{fontenc}    % use 8-bit T1 fonts
\usepackage{hyperref}       % hyperlinks
\usepackage{url}            % simple URL typesetting
\usepackage{booktabs}       % professional-quality tables
\usepackage{amsfonts}       % blackboard math symbols
\usepackage{nicefrac}       % compact symbols for 1/2, etc.
\usepackage{microtype}      % microtypography
\usepackage{algorithm}
\usepackage{algorithmic}
\usepackage{graphicx}
\usepackage{caption}
\usepackage{multirow}
\usepackage{color}
\usepackage{amsmath}\allowdisplaybreaks
\usepackage{amsfonts,bm}
\usepackage{amssymb}

\def\eqref#1{equation~(\ref{#1})}
\def\1{\bf{1}}

\newcommand{\norm}[1]{\left\| #1 \right\|_2}
\def\inner#1#2{\langle #1, #2 \rangle}

\def\eps{{\varepsilon}}

\def\va{{\bf{a}}}
\def\vb{{\bf{b}}}
\def\vc{{\bf{c}}}

\def\vu{{\bf{u}}}
\def\vv{{\bf{v}}}
\def\vw{{\bf{w}}}
\def\vx{{\bf{x}}}
\def\vy{{\bf{y}}}
\def\vz{{\bf{z}}}

\def\fB{{\mathcal{B}}}
\def\fC{{\mathcal{C}}}

\def\fG{{\mathcal{G}}}

\def\fO{{\mathcal{O}}}

\def\fS{{\mathcal{S}}}

\def\fV{{\mathcal{V}}}

\def\fY{{\mathcal{Y}}}

\def\BE{{\mathbb{E}}}
\def\BR{{\mathbb{R}}}

\def\mI {{\bf I}}

\newcommand{\R}{\mathbb{R}}

\DeclareMathOperator*{\argmax}{arg\,max}
\DeclareMathOperator*{\argmin}{arg\,min}

\usepackage{amsthm}
\theoremstyle{plain}
\newtheorem{thm}{Theorem}%[section]
\newtheorem{dfn}{Definition}

\newtheorem{lem}{Lemma}
\newtheorem{asm}{Assumption}

\newtheorem{cor}{Corollary}

\makeatletter
\def\Ddots{\mathinner{\mkern1mu\raise\p@
\vbox{\kern7\p@\hbox{.}}\mkern2mu
\raise4\p@\hbox{.}\mkern2mu\raise7\p@\hbox{.}\mkern1mu}}
\makeatother

\makeatletter
\newcommand*{\rom}[1]{\expandafter\@slowromancap\romannumeral #1@}
\makeatother

\def\hu{{\hat{\bf{u}}}}
\def\tw{{\tilde{\bf{w}}}}
\def\tx{{\tilde{\bf{x}}}}
\def\ty{{\tilde{\bf{y}}}}
\def\tv{{\tilde{\bf{v}}}}
\def\tu{{\tilde{\bf{u}}}}
\def\vxi{{\bm{\xi}}}
\def\tDelta{{\widetilde{\Delta}}}

\usepackage[numbers,sort&compress]{natbib}
\title{Stochastic Recursive Gradient Descent Ascent for Stochastic Nonconvex-Strongly-Concave Minimax Problems}

\author{
    Luo Luo$^1$ \qquad Haishan Ye$^2$ \qquad Zhichao Huang$^1$ \qquad Tong Zhang$^1$ \\ 
    $^1$Department of Mathematics, The Hong Kong University of Science and Technology \\
    $^2$Shenzhen Research Institute of Big Data, The Chinese University of Hong Kong, Shenzhen  \\
    {\small \texttt{luoluo@ust.hk}\hskip1.8em
    \texttt{hsye\_cs@outlook.com}\hskip1.8em
    \texttt{zhuangbx@connect.ust.hk}\hskip1.8em
    \texttt{tongzhang@ust.hk}}
}

\begin{document}

\maketitle

\begin{abstract}
We consider nonconvex-concave minimax optimization problems of the form $\min_{\bf x}\max_{\bf y\in{\mathcal Y}} f({\bf x},{\bf y})$, where $f$ is strongly-concave in $\bf y$ but possibly nonconvex in $\bf x$ and ${\mathcal Y}$ is a convex and compact set. We focus on the stochastic setting, where we can only access an unbiased stochastic gradient estimate of $f$ at each iteration. This formulation includes many machine learning applications as special cases such as robust optimization and adversary training. We are interested in finding an ${\mathcal O}(\varepsilon)$-stationary point of the function $\Phi(\cdot)=\max_{\bf y\in{\mathcal Y}} f(\cdot, {\bf y})$. The most popular algorithm to solve this problem is stochastic gradient decent ascent, which requires $\mathcal O(\kappa^3\varepsilon^{-4})$ stochastic gradient evaluations, where $\kappa$ is the condition number. In this paper, we propose a novel method called Stochastic Recursive gradiEnt Descent Ascent (SREDA), which estimates gradients more efficiently using variance reduction. This method achieves the best known stochastic gradient complexity of ${\mathcal O}(\kappa^3\varepsilon^{-3})$, and its dependency on $\varepsilon$ is optimal for this problem. 
\end{abstract}

\section{Introduction}

This paper considers the following minimax optimization problem
\begin{align}
    \min_{\vx\in\BR^{d}}\max_{\vy\in\fY} f(\vx,\vy) \triangleq {\mathbb E}\left[F({\bf x},{\bf y}; \vxi)\right], \label{prob:main-f}
\end{align}
where the stochastic component $F(\vx, \vy; \vxi)$,
indexed by some random vector $\vxi$,  is $\ell$-gradient Lipschitz on average.
This minimax optimization formulation includes many machine learning applications such as regularized empirical risk minimization~\cite{zhang2017stochastic,tan2018stochastic}, AUC maximization~\cite{ying2016stochastic,shen2018towards}, robust optimization~\cite{duchi2019variance,yan2019stochastic}, adversarial training~\cite{goodfellow2014generative,goodfellow2014explaining,sinha2017certifying} and reinforcement learning \cite{wai2018multi,du2017stochastic}.
Many existing work~\cite{chambolle2011first,ying2016stochastic,palaniappan2016stochastic,zhang2017stochastic,tan2018stochastic,chavdarova2019reducing,
ouyang2018lower,zhang2019lower,du2017stochastic,du2018linear,luo2019stochastic,hsieh2019convergence,xie2020lower}
focused on the convex-concave case of problem (\ref{prob:main-f}),
where $f$ is convex in $\vx$ and concave in $\vy$. For such problems,
one can establish strong theoretical guarantees.

In this paper, we focus on a more general case of (\ref{prob:main-f}),
where $f(\vx,\vy)$ is $\mu$-strongly-concave in $\vy$ but possibly
nonconvex in $\vx$.  This case is referred to as stochastic
nonconvex-strongly-concave minimax problems, and it is equivalent to the following problem
\begin{align}
    \min_{\vx\in\R^{d}}\left\{\Phi(\vx) \triangleq \max_{\vy\in\fY} f(\vx,\vy)\right\}. \label{prob:main-phi}
\end{align}
Formulation (\ref{prob:main-phi}) contains several interesting examples in machine 
 learning such as
robust optimization~\cite{duchi2019variance,yan2019stochastic} and
adversarial training~\cite{goodfellow2014explaining,sinha2017certifying}.

Since $\Phi$ is possibly nonconvex, it is infeasible to find the global minimum in general.
One important task of the minimax problem is finding an approximate stationary point of $\Phi$.
A simple way to solve this problem is stochastic gradient descent with max-oracle (SGDmax) \cite{jin2019local,lin2019gradient}.
The algorithm includes a nested loop to solve $\max_{\vy\in\fY} f(\vx,\vy)$ and use the solution to run approximate stochastic gradient descent (SGD) on $\vx$. \citet{lin2019gradient} showed that we can solve problem (\ref{prob:main-phi}) by directly extending SGD to stochastic gradient descent ascent (SGDA).
The iteration of SGDA is just using gradient descent on $\vx$ and
gradient descent on $\vy$. The complexity of SGDA to
find $\fO(\eps)$-stationary point of $\Phi$ in expectation is
$\fO\left(\kappa^3\eps^{-4}\right)$ stochastic gradient evaluations, where $\kappa\triangleq\ell/\mu$ is the condition number. SGDA is more efficient than SGDmax whose complexity is $\fO\left((\kappa^3\eps^{-4})\log(1/\eps)\right)$.

One insight of SGDA is that the algorithm selects an appropriate ratio
of learning rates for $\vx$ and $\vy$. Concretely, the learning rate
for updating $\vy$ is $\fO(\kappa^2)$ times that of $\vx$.  Using this
idea, it can be shown that the nested loop of SGDmax is unnecessary,
and SGDA eliminates the logarithmic term in the complexity result. 
In addition, \citet{rafique2018non}
presented some nested-loop algorithms that also achieved $\fO\left(\kappa^3\eps^{-4}\right)$ complexity. Recently, \citet{yan2020sharp} proposed Epoch-GDA which considered constraints on both two variables.

\citet{lin2020near} proposed a deterministic algorithm called minimax proximal point algorithm (Minimax PPA) to solve nonconvex-strongly-concave minimax problem whose complexity has square root dependence on $\kappa$. \citet{thekumparampil2019efficient,barazandeh2020solving,ostrovskii2020efficient} also studied the non-convex-concave minimax problems, however, these methods do not cover the stochastic setting in this paper and only work for a special case of problem (\ref{prob:main-phi}) when the stochastic variable $\vxi$ is finitely sampled from $\{\vxi_1,\dots,\vxi_n\}$ (a.k.a. finite-sum case). That is, 
\begin{align}
    f(\vx,\vy)\triangleq \frac{1}{n}\sum_{i=1}^n F(\vx,\vy;\vxi_i). \label{prob:main-finite}
\end{align}

In this paper, we propose a novel algorithm called Stochastic Recursive gradiEnt Descent Ascent (SREDA) for stochastic nonconvex-strongly-concave minimax problems.
Unlike SGDmax and SGDA, which only iterate with current stochastic gradients, our SREDA updates the estimator recursively and reduces its variance. 

The variance reduction techniques have been widely used in convex and nonconvex minimization 
problems~\cite{zhang2013linear,johnson2013accelerating,nguyen2017sarah,schmidt2017minimizing,defazio2014saga,allen2017katyusha,allen2018katyushax,nguyen2018inexact,reddi2016stochastic,allen2017natasha,lei2017non,allen2016variance,fang2018spider,pham2019proxsarah,li2018simple,zhou2018stochastic,wang2019spiderboost} and convex-concave saddle point problems~\cite{palaniappan2016stochastic,du2017stochastic,du2018linear,luo2019stochastic,chavdarova2019reducing}.
However, the nonconvex-strongly-concave minimax problems have two variables $\vx$ and $\vy$ and their roles in the objective function are quite different. To apply the  technique of variance reduction, SREDA employs a concave maximizer with multi-step iteration on $\vy$ to simultaneously balance the learning rates, gradient batch sizes and iteration numbers of the two variables. We prove SREDA reduces the number of stochastic gradient evaluations to $\fO(\kappa^3\eps^{-3})$, which is the best known upper bound complexity. The result gives optimal dependency on $\eps$ since the lower bound of stochastic first order algorithms for general nonconvex optimization is $\fO(\eps^{-3})$~\cite{ACDFSW2019}.
For finite-sum cases, the gradient cost of SREDA is $\fO\left(n\log(\kappa/\eps) + \kappa^2n^{1/2}\eps^{-2}\right)$ when $n\geq \kappa^2$, and $\fO\left((\kappa^2 + \kappa n)\eps^{-2} \right)$ when $n\leq \kappa^2$. This result is sharper than Minimax PPA \cite{lin2020near} in the case of $n$ is larger than  $\kappa^2$. We summarize the comparison of all algorithms in Table \ref{table:complexity}.

\begin{table*}
\centering
\caption{We present the comparison on stochastic gradient complexities of  algorithms to solve stochastic problem (\ref{prob:main-phi}) and finite-sum problem (\ref{prob:main-finite}). We use notation $\tilde\fO(\cdot)$ to hide logarithmic factors. Some baseline algorithms solve problem (\ref{prob:main-finite}) without considering the finite-sum structure and we regard the cost of full gradient evaluation is $\fO(n)$.}\label{table:complexity}
\newcommand{\tabincell}[2]{\begin{tabular}{@{}#1@{}}#2\end{tabular}}
\renewcommand{\arraystretch}{1.5}
\begin{tabular}{|c|c|c|c|}
\hline
Algorithm & Stochastic & Finite-sum & Reference\\ \hline
SGDmax (GDmax)       & $\tilde\fO(\kappa^3\eps^{-4})$ & $\tilde\fO(\kappa^2n\eps^{-2})$ & \cite{jin2019local,lin2019gradient} \\\hline
PGSMD / PGSVRG & $\fO(\kappa^3\eps^{-4})$ & $\fO(\kappa^2n\eps^{-2})$ & \cite{rafique2018non}\\  \hline
MGDA / HiBSA   & -- & $\fO(\kappa^4n\eps^{-2})$  & \cite{nouiehed2019solving,lu2019hybrid} \\ \hline
Minimax PPA & -- & $\tilde\fO(\kappa^{1/2}n\eps^{-2})$ & \cite{lin2020near} \\ \hline
SGDA (GDA)   & $\fO(\kappa^3\eps^{-4})$ & $\fO(\kappa^2n\eps^{-2})$ & \cite{lin2019gradient} \\ \hline
SREDA   & $\fO(\kappa^3\eps^{-3})$ &
\tabincell{c}{ \\[-0.5cm]
$\begin{cases}
{\tilde\fO\left(n + \kappa^2n^{1/2}\eps^{-2}\right)}, & n\geq \kappa^2 \\[0.05cm]
\fO\left((\kappa^2 + \kappa n)\eps^{-2} \right), & n\leq \kappa^2
\end{cases}$ \\[0.25cm]}
& this paper \\ \hline
\end{tabular}
\renewcommand{\arraystretch}{1}
\end{table*}

The paper is organized as follows.
In Section \ref{section:preliminaries}, we present notations and preliminaries.
In Section \ref{section:related}, we review the existing work for stochastic nonconvex-strongly-concave optimization and related techniques.
In Section \ref{section:SREDA}, we present the SREDA algorithm and the main theoretical result.
In Section \ref{section:proofs}, we give a brief overview of our convergence analysis. 
In Section \ref{section:experiment}, we demonstrate the effectiveness of our methods on robust optimization problem. We conclude this work in Section \ref{section:conclusions}.

\section{Notation and Preliminaries}\label{section:preliminaries}

We first introduce the notations and preliminaries used in this paper.
For a differentiable function $f(\vx,\vy)$,
we denote the partial gradient of $f$ with respect to $\vx$ and $\vy$ at $(\vx,\vy)$ as $\nabla_\vx f(\vx,\vy)$ and $\nabla_\vy f(\vx,\vy)$ respectively. We use $\norm{\cdot}$ to denote the Euclidean norm of vectors. For a finite set $\fS$, we denote its cardinality as $|\fS|$.
We assume that the minimax problem (\ref{prob:main-phi}) satisfies the following assumptions.
\begin{asm}\label{asm:bound}
    The function $\Phi(\cdot)$ is lower bounded, i.e., we have $\Phi^*=\inf_{\vx\in\BR^{d}} \Phi(\vx) > -\infty$.
\end{asm}

\begin{asm}\label{asm:F}
    The component function $F$ has an average $\ell$-Lipschitz gradient,  i.e., there exists a constant $\ell>0$ such that 
    $\BE\norm{\nabla F(\vx,\vy; \vxi) - \nabla F(\vx',\vy'; \vxi)}^2 \leq \ell^2(\norm{\vx-\vx'}^2 + \norm{\vy-\vy'}^2)$
    for any $(\vx,\vy)$, $(\vx',\vy')$ and random vector $\vxi$
\end{asm}

\begin{asm}\label{asm:F2}
    The component function $F$ is concave in $\vy$. That is, for any $\vx$, $\vy$, $\vy'$  and random vector $\vxi$, we have
    $F(\vx,\vy; \vxi) \leq F(\vx,\vy'; \vxi) + \inner{\nabla_\vy F(\vx,\vy';\vxi)}{\vy-\vy'}$.
\end{asm}

\begin{asm}\label{asm:f}
    The function $f(\vx,\vy)$ is $\mu$-strongly-concave in $\vy$. That is,
    there exists a constant $\mu>0$ such that for any $\vx$, $\vy$ and $\vy'$, we have
    $f(\vx,\vy) \leq f(\vx,\vy') + \inner{\nabla_\vy f(\vx,\vy')}{\vy-\vy'} - \frac{\mu}{2}\norm{\vy-\vy'}^2$.
\end{asm}

\begin{asm}\label{asm:grad}
    The gradient of each component function $F(\vx,\vy;\vxi)$ has bounded variance. That is, there exists a constant $\sigma>0$ such that $\BE\norm{\nabla F(\vx,\vy;\vxi) - \nabla f(\vx,\vy)}^2 \leq \sigma^2 < \infty$
    for any $\vx$, $\vy$ and random vector $\vxi$.
\end{asm}
Under the assumptions of Lipschitz-gradient and strongly-concavity on $f$, we can show that $\Phi(\cdot)$ also has Lipschitz-gradient.
\begin{lem}[{\cite[Lemma 4.3]{lin2019gradient}}]\label{lem:Phi-smooth}
    Under Assumptions \ref{asm:F} and \ref{asm:f}, the function $\Phi(\cdot)=\max_{\vy\in\fY} f(\cdot,\vy)$ has $(\ell+\kappa\ell)$-Lipschitz gradient.
    Additionally, the function $\vy^*(\cdot)=\argmax_\vy f(\cdot,\vy)$ is unique defined and we have $\nabla\Phi(\cdot)=\nabla_\vx f(\cdot,\vy^*(\cdot))$.
\end{lem}

Since $\Phi$ is differentiable, we may define $\eps$-stationary point based on its gradient.
The goal of this paper is to establish a stochastic gradient algorithm that output an $\fO(\eps)$-stationary point in expectation.
\begin{dfn}
    We call $\vx$ an $\fO(\eps)$-stationary point of $\Phi$ if $\norm{\nabla \Phi(\vx)}\leq \fO(\eps)$.
\end{dfn}
We also need the notations of projection and gradient mapping to address the constraint on $\fY$.
\begin{dfn}
    We define the projection of $\vy$ on to convex set $\fY$ by $\Pi_\fY(\vy)=\argmin_{\vz\in\fY}\norm{\vz-\vy}$.
\end{dfn}
\begin{dfn}
    We define the gradient mapping of $f$ at $(\vx',\vy')$ with respect to $\vy$ as follows
    \begin{align*}
        \fG_{\lambda,\vy}(\vx',\vy')=\frac{1}{\lambda}\left(\vy'-\Pi_\fY\left(\vy'+\lambda\nabla_\vy f(\vx',\vy')\right)\right),
        \text{~~where~} \lambda>0. 
    \end{align*}
\end{dfn}

\section{Related Work}\label{section:related}

In this section, we review recent works for solving stochastic nonconvex-strongly-convex minimax problem (\ref{prob:main-phi}) and introduce variance reduction techniques in stochastic optimization.

\subsection{Nonconvex-Strongly-Concave Minimax}
We present SGDmax~\cite{jin2019local,lin2019gradient} in Algorithm \ref{alg:SGDmax}.
We can realize the max-oracle by stochastic gradient ascent (SGA) with $\fO(\kappa^2\eps^{-2}\log(1/\eps))$ stochastic gradient evaluations to achieve sufficient accuracy.  Using $S=\fO(\kappa\eps^{-2})$ guarantees that the variance of the stochastic gradients is less than $\fO(\kappa^{-1}\eps^2)$. It requires $\fO(\kappa\eps^{-2})$ iterations with step size $\eta=\fO(1/(\kappa\ell))$ to obtain an $\fO(\eps)$-stationary point of $\Phi$.
The total stochastic gradient evaluation complexity is $\fO(\kappa^3\eps^{-4}\log(1/\eps))$.
The procedure of SGDA is shown in Algorithm \ref{alg:SGDA}.

Since variables $\vx$ and $\vy$ are not symmetric, we need to select different step sizes for them. 
In our case, we choose $\eta=\fO(1/(\kappa^2\ell))$ and $\lambda=\fO(1/\ell)$.
This leads to an $\fO(\kappa^3\eps^{-4})$ complexity to obtain an $\fO(\eps)$-stationary point with $S=\fO(\kappa\eps^{-2})$ and $\fO(\kappa^2\eps^{-2})$ iterations~\cite{lin2019gradient}. \citeauthor{rafique2018non} proposed proximally guided stochastic mirror descent and variance reduction (PGSMD / PGSVRG) whose complexity is also $\fO(\kappa^3\eps^{-4})$.
Both of the above algorithms reveal that the key of solving problem (\ref{prob:main-phi}) efficiently is to update $\vy$ much more frequently than $\vx$. The natural intuition is that finding stationary point of a nonconvex function is typically more difficult than finding that of a concave or convex function. SGDmax implements it by updating $\vy$ more frequently (SGA in max-oracle) while SGDA iterates $\vy$ with a larger step size such that $\lambda/\eta = \fO(\kappa^2)$.

\subsection{Variance Reduction Techniques}
Variance reduction techniques has been widely used in stochastic optimization
~\cite{nguyen2017sarah,nguyen2018inexact,reddi2016stochastic,allen2017natasha,lei2017non,allen2016variance,fang2018spider,pham2019proxsarah,li2018simple}.
One scheme of this type of methods is StochAstic Recursive grAdient algoritHm (SARAH) \cite{nguyen2017sarah,nguyen2018inexact}.
\citet{nguyen2017sarah} first proposed it for convex minimization and established a convergence result.
For nonconvex optimization, a closely related method is Stochastic Path-Integrated Differential EstimatoR (SPIDER)~\cite{fang2018spider}.
The algorithm estimates the gradient recursively together with a normalization rule, which guarantees the approximation error of the gradient is $\fO(\eps^2)$ at each step.
As a result, it can find $\fO(\eps)$-stationary point of the nonconvex objective in $\fO(\eps^{-3})$ complexity, which matches the lower bound \cite{ACDFSW2019}.
This idea can also be extended to nonsmooth cases~\cite{pham2019proxsarah,wang2019spiderboost}.

It is also possible to employ variance reduction to solve minimax problems. 
Most of the existing works focused on the convex-concave case.
For example, \citet{palaniappan2016stochastic,chavdarova2019reducing}, extend SVRG~\cite{zhang2013linear,johnson2013accelerating} and SAGA~\cite{defazio2014saga} to solving strongly-convex-strongly-concave minimax problem in the finite-sum case, and established a linear convergence.
One may also use the Catalyst framework~\cite{palaniappan2016stochastic,lin2018catalyst} 
and proximal point iteration~\cite{luo2019stochastic,defazio2016simple} to further accelerate when the problem is ill-conditioned.
\citet{du2017stochastic,du2018linear} pointed out that for some special cases, the strongly-convex and strongly-concave assumptions of linear convergence for minimax problem may not be necessary.
Additionally, \citet{zhang2019multi} solved multi-level composite optimization problems by variance reduction, but the oracles in their algorithms are different from our settings.  

\begin{algorithm}[ht]
    \caption{SGDmax}\label{alg:SGDmax}
	\begin{algorithmic}[1]
	\STATE \textbf{Input} initial point $\vx_0$, learning rate $\eta>0$, batch size $S>0$, max-oracle accuracy $\zeta$ \\[0.1cm]
    \STATE \textbf{for}  $k=0,\dots,K$ \textbf{do} \\[0.1cm]
    \STATE\quad draw $S$ samples $\{\vxi_1,\dots,\vxi_{S}\}$ \\[0.1cm]
    \STATE\quad find $\vy_k$ so that $\BE[f(\vx_k,\vy_k)]\geq \max_{\vy\in\fY} f(\vx_k,\vy)-\zeta$ \\[0.1cm]
    \STATE\quad $\vx_{k+1}=\vx_k - \eta \cdot \frac{1}{S}\sum_{i=1}^{S} \nabla_\vx  F(\vx_k,\vy_k;\vxi_i)$ \\[0.1cm]
    \STATE \textbf{end for} \\[0.1cm]
	\STATE \textbf{Output} $\hat\vx$ chosen uniformly at random from $\{\vx_i\}_{i=0}^K$
	\end{algorithmic}
\end{algorithm}

\begin{algorithm}[ht]
    \caption{SGDA}\label{alg:SGDA}
	\begin{algorithmic}[1]
	\STATE \textbf{Input} initial point $(\vx_0, \vy_0)$, learning rates $\eta>0$ and $\lambda>0$, batch size $S>0$ \\[0.1cm]
    \STATE \textbf{for}  $k=0,\dots,K$ \textbf{do} \\[0.1cm]
    \STATE\quad draw $M$ samples $\{\vxi_1,\dots,\vxi_{S}\}$ \\[0.1cm]
    \STATE\quad $\vx_{k+1}=\vx_k - \eta \cdot \frac{1}{S}\sum_{i=1}^{S} \nabla_\vx F(\vx_k,\vy_k;\vxi_i)$ \\[0.1cm]
    \STATE\quad $\vy_{k+1}=\Pi_\fY\left(\vy_k + \lambda \cdot \frac{1}{S}\sum_{i=1}^{S} \nabla_\vy F(\vx_k,\vy_k;\vxi_i)\right)$ \\[0.1cm]
    \STATE \textbf{end for} \\[0.1cm]
	\STATE \textbf{Output} $\hat\vx$ chosen uniformly at random from $\{\vx_i\}_{i=0}^K$
	\end{algorithmic}
\end{algorithm}

\begin{algorithm}[ht]
    \caption{SREDA}\label{alg:SREDA}
	\begin{algorithmic}[1]
	\STATE \textbf{Input} initial point $\vx_0$, learning rates $\eta_k,\lambda>0$, batch sizes $S_1,S_2>0$; periods $q,m>0$, number of initial iterations $K_0$ \\[0.1cm]
    \STATE\label{line:PiSARAH} $\vy_0=\text{PiSARAH}\left(-f(\vx_k,\cdot),~K_0\right)$ \\[0.1cm]
    \STATE \textbf{for}  $k=0,\dots,K-1$ \textbf{do} \\[0.1cm]
    \STATE\quad \textbf{if} $\mod(k, q)=0$ \\[0.1cm]
    \STATE\quad\quad\label{line:S1-1} draw $S_1$ samples $\{\vxi_1,\dots,\vxi_{S_1}\}$ \\[0.1cm]
    \STATE\quad\quad\label{line:S1-2} $\vv_k=\frac{1}{S_1}\sum_{i=1}^{S_1} \nabla_\vx F(\vx_k,\vy_k;\vxi_i)$ \\[0.1cm]
    \STATE\quad\quad\label{line:S1-3} $\vu_k=\frac{1}{S_1}\sum_{i=1}^{S_1} \nabla_\vy F(\vx_k,\vy_k;\vxi_i)$ \\[0.1cm]
    \STATE\quad\textbf{else} \\[0.1cm]
    \STATE\label{line:vv}\quad\quad $\vv_k=\vv'_{k}$ \\[0.1cm]
    \STATE\label{line:vu}\quad\quad $\vu_k=\vu'_{k}$ \\[0.1cm]
    \STATE\quad \textbf{end if}\\[0.1cm]
    \STATE\quad $\vx_{k+1}=\vx_k - \eta_k \vv_k$ \\[0.1cm]
    \STATE\quad\label{line:inner} $(\vy_{k+1},\vv'_{k+1},\vu'_{k+1})=\text{ConcaveMaximizer}\left(k, m, S_2, \vx_k, \vx_{k+1}, \vy_k, \vu_k, \vv_k \right)$ \\[0.1cm]
    \STATE \textbf{end for} \\[0.1cm]
	\STATE \textbf{Output} $\hat\vx$ chosen uniformly at random from $\{\vx_i\}_{i=0}^{K-1}$
	\end{algorithmic}
\end{algorithm}

\section{Algorithms and Main Results}\label{section:SREDA}
In this section, we propose a novel algorithm for solving problem (\ref{prob:main-phi}), which we 
call Stochastic Recursive gradiEnt Descent Ascent (SREDA).
We show that the algorithm finds an $\fO(\eps)$-stationary point with
a complexity of $\fO(\kappa^3\eps^{-3})$ stochastic gradient evaluations,
and this result may be extended to the finite-sum case (\ref{prob:main-finite}).

\subsection{Stochastic Recursive Gradient Descent Ascent}

SREDA uses variance reduction to track the gradient estimator recursively.
Because there are two variables $\vx$ and $\vy$ in our problem (\ref{prob:main-phi}), it is not efficient to combine SGDA with SPIDER~\cite{fang2018spider} or (inexact) SARAH~\cite{nguyen2017sarah,nguyen2018inexact} directly.
The algorithm should approximate the gradient of $f(\vx_k,\vy_k)$ with small error,
and keep the value of $f(\vx_k,\vy_k)$ sufficiently close to $\Phi(\vx_k)$.
To achieve this, in the proposed method SREDA,  we employ a concave maximizer with stochastic variance reduced gradient ascent on $\vy$.
The details of SREDA and the concave maximizer are presented in Algorithm \ref{alg:SREDA} and Algorithm \ref{alg:inner} respectively.
In the rest of this section, we show SREDA can find an $\fO(\eps)$-stationary point in $\fO(\kappa^3\eps^{-3})$ stochastic gradient evaluations.

In the initialization of SREDA, we hope to obtain  $\vy_0\approx\arg\max_{\vy\in\fY}f(\vx_0,\vy_0)$ for given $\vx_0$ such that $\BE\norm{\fG_{\lambda,\vy} (\vx_0,\vy_0)}^2 \leq \fO(\kappa^{-2}\eps^2)$ . We proposed a new algorithm called projected inexact SARAH (PiSARAH) to address it. PiSARAH extends inexact SARAH (iSARAH)~\cite{nguyen2018inexact} to constrained case, which could achieve the desired accuracy of our initialization with a complexity of $\fO(\kappa^2\eps^{-2}\log(\kappa/\eps))$. We present the details of PiSARAH in Appendix C.

%In the rest of this section, we show SREDA can keep the gradient with respect to $\vy$ less than $\fO(\kappa^{-2}\eps^2)$ in expectation, and obtain an $\fO(\eps)$-stationary point in $\fO(\kappa^3\eps^{-3})$ stochastic gradient evaluations.

SREDA estimates the gradient of $f(\vx_k,\vy_k)$ by
$(\vv_k, \vu_k) \approx \left(\nabla_\vx f(\vx_k,\vy_k), \nabla_\vy f(\vx_k,\vy_k)\right)$.
As illustrated in Algorithm \ref{alg:inner},
we evaluate the gradient of $f$ with a large batch size $S_1=\fO(\kappa^2\eps^{-2})$ at the beginning of each period, and update the gradient estimate recursively in concave maximizer with a smaller batch size $S_2=\fO(\kappa\eps^{-1})$.
%Considering the roles of $\vx_k$ and $\vy_k$ are asymmetric in problem (\ref{prob:main-phi}), we should iterates them in different way.

For variable $\vx_k$, we adopt a normalized stochastic gradient descent with a learning rate for theoretical analysis:
{\small\begin{align*}
    \eta_k=\min\left(\dfrac{\eps}{\ell\norm{\vv_k}}, \frac{1}{2\ell}\right)\cdot\fO(\kappa^{-1}).
\end{align*}}
With this step size, the change of $\vx_k$ is not dramatic at each iteration, which leads to accurate gradient estimates. To simplify  implementations of the algorithm, we can also use a fixed learning rate in practical.
%When $\vv_k$ is large, we have $\eta_k=\fO(1/(\kappa\ell))$ which is larger than the stepsize $\eta_k=\fO(1/(\kappa^2\ell))$ of SGDA~\cite{lin2019gradient}.

For variable $\vy_k$, we additionally expect $f(\vx_k, \vy_k)$ is a good approximation of $\Phi(\vx_k)$, which implies the gradient mapping with respect to $\vy_k$ should be small enough. We hope to maintain the inequality $\BE\norm{\fG_{\lambda,\vy} (\vx_k,\vy_k)}^2 \leq \fO(\kappa^{-2}\eps^2)$ holds.
Hence, we include a multi-step concave maximizer to update $\vy$ whose details given in Algorithm \ref{alg:inner}.
This procedure can be regarded as one epoch of PiSARAH. We choose the step size $\lambda=\fO(1/\ell)$ for inner iterations, which simultaneously ensure that the gradient mapping with respect to $\vy$ is small enough and the change of $\vy$ is not dramatic. %The update rule of concave maximizer is also based on stochastic recursive gradient that is friendly to our later analysis.

\begin{algorithm*}[t]
    \caption{${\rm ConcaveMaximizer}~(k, m, S_2, \vx_k, \vx_{k+1}, \vy_k, \vu_k, \vv_k)$}
	\label{alg:inner}
	\begin{algorithmic}[1]
	\STATE \textbf{Initialize} $\tx_{k,-1}=\vx_{k}$, $\ty_{k,-1}=\vy_{k}$, $\tx_{k,0}=\vx_{k+1}$, $\ty_{k,0}=\vy_k$. \\[0.1cm]
	%\STATE \textbf{Input} $\tx_{k,-1}$, $\ty_{k,-1}$, $\tx_{k,0}$, $\ty_{k,0}$ and number of iterations $m$. \\[0.1cm]
    \STATE draw $S_2$ samples $\{\vxi_{1},\dots,\vxi_{S_2}\}$\\[0.1cm]
    \STATE $\tv_{k,0}=\vv_{k}+\frac{1}{S_2}\sum_{i=1}^{S_2} \nabla_\vx F(\tx_{k,0},\ty_{k,0};\vxi_{i})-\frac{1}{S_2}\sum_{i=1}^{S_2} \nabla_\vx F(\tx_{k,-1},\ty_{k,-1};\vxi_{i})$ \\[0.1cm]
    \STATE $\tu_{k,0}=\vu_{k}+\frac{1}{S_2}\sum_{i=1}^{S_2} \nabla_\vy F(\tx_{k,0},\ty_{k,0};\vxi_{i})-\frac{1}{S_2}\sum_{i=1}^{S_2} \nabla_\vy F(\tx_{k,-1},\ty_{k,-1};\vxi_{i})$ \\[0.1cm]
    \STATE $\tx_{k,1}=\tx_{k,0}$ \\[0.1cm]
    \STATE $\ty_{k,1}=\ty_{k,0} + \lambda\tu_{k,0}$ \\[0.1cm]
    \STATE \textbf{for}  $t=1,\dots,m+1$ \textbf{do} \\[0.1cm]
    \STATE \quad draw $S_2$ samples $\{\vxi_{t,1},\dots,\vxi_{t,S_2}\}$ \\[0.1cm]
    \STATE\label{line:inner-v} \quad $\tv_{k,t}=\tv_{k,t-1}+\frac{1}{S_2}\sum_{i=1}^{S_2} \nabla_\vx F(\tx_{k,t},\ty_{k,t};\vxi_{t,i})-\frac{1}{S_2}\sum_{i=1}^{S_2} \nabla_\vx F(\tx_{k,t-1},\ty_{k,t-1};\vxi_{t,i})$ \\[0.1cm]
    \STATE\label{line:inner-u} \quad $\tu_{k,t}=\tu_{k,t-1}+\frac{1}{S_2}\sum_{i=1}^{S_2} \nabla_\vy F(\tx_{k,t},\ty_{k,t};\vxi_{t,i})-\frac{1}{S_2}\sum_{i=1}^{S_2} \nabla_\vy F(\tx_{k,t-1},\ty_{k,t-1};\vxi_{t,i})$ \\[0.1cm]
    \STATE\quad $\tx_{k,t+1} = \tx_{k,t}$ \\[0.1cm]
    \STATE\quad $\ty_{k,t+1} = \Pi_\fY\left(\ty_{k,t} + \lambda \tu_{k,t}\right)$ \label{line:proj_y}\\[0.1cm]
    \STATE \textbf{end for} \\[0.1cm]
	\STATE\label{line:inner-out} \textbf{Output} %$\vv'_{k} = \tv_{k,s_k}$, $\vu'_{k} = \tu_{k,s_k}$,
        %$\vy_{k+1} = \ty_{k,s_k+1}$  and $s_k$ is sampled from $\{0,1,\dots,m\}$
        %$\vv'_{k+1} = \ty_{k,s_k+1}$,
       % $\vu'_{k+1} = \ty_{k,s_k+1}$,
        $\ty_{k,s_k}$, $\tv_{k,s_k}$ and $\tu_{k,s_k}$ where $s_k$ is chosen uniformly at random from $\{1,\dots,m\}$
	\end{algorithmic}
\end{algorithm*}

\subsection{Complexity Analysis}\label{subsection:complexity}

As shown in Algorithm \ref{alg:SREDA},  SREDA updates variables with a large batch size per $q$ iterations. 
We choose $q=\fO(\eps^{-1})$ as a balance between the number of large batch evaluations with $S_1=\fO(\kappa^2\eps^{-2})$ samples and the concave maximizer with $\fO(\kappa)$ iterations and $S_2=\fO(\kappa\eps^{-1})$ samples. 

Based on above parameter setting, we can obtain an approximate stationary point $\hat\vx$ in expectation such that
$\BE\norm{\nabla\Phi(\hat\vx)}\leq\fO(\eps)$ with $K=\fO(\kappa\eps^{-2})$ outer iterations.
The total number of stochastic gradient evaluations of SREDA comes from the initial run of PiSARAH, large batch gradient evaluation ($S_1$ samples) and concave maximizer. That is,
%\begin{align*}
% &~{\rm StocGrad~(Total)} \\
%=&~\fO(\kappa^3\eps^{-2}\log(\kappa/\eps))+\fO\left(\frac{K}{q}\cdot S_1\right) 
%    + \fO\left(K\cdot S_2\cdot m\right) \\[0.1cm]
%=&~\fO(\kappa^3\eps^{-2}\log(\kappa/\eps))+\fO\left(\frac{\kappa\eps^{-2}}{\eps^{-1}}\cdot \kappa^2\eps^{-2}\right) 
%    + \fO\left(\kappa\eps^{-2} \cdot \kappa\eps^{-1}\cdot \kappa\right) \\ 
%= & \fO(\kappa^3\eps^{-3}).
%\end{align*}
\begin{align*}
 \fO(\kappa^2\eps^{-2}\log(\kappa/\eps))+\fO\left(K/q\cdot S_1\right) + \fO\left(K\cdot S_2\cdot m\right) 
= \fO(\kappa^3\eps^{-3}).
\end{align*}
Let $\Delta_f=f(\vx_0,\vy_0)+\frac{134\eps^2}{\kappa\ell}-\Phi^*$, then we formally present the main result in Theorem \ref{thm:main}.
\begin{thm}\label{thm:main}
    Under Assumptions \ref{asm:bound}-\ref{asm:grad} with the following parameter choices:
    \begin{align*}
        & \zeta = \kappa^{-2}\eps^2,~
        \eta_k=\min\left(\dfrac{\eps}{5\kappa\ell\norm{\vv_k}}, \dfrac{1}{10\kappa\ell}\right),~
        \lambda=\dfrac{1}{8\ell},~  
          S_1=\left\lceil\frac{2250}{19}\sigma^2\kappa^{-2}\eps^2\right\rceil, \\[0.1cm]
          & S_2=\left\lceil\frac{3687}{76}\kappa q\right\rceil,~
          q = \left\lceil\eps^{-1}\right\rceil,~
          K = \left\lceil\dfrac{100\kappa\ell \eps^{-2}\Delta_f}{9}\right\rceil~\text{and}~
        m=\lceil1024\kappa\rceil,
    \end{align*}
    Algorithm \ref{alg:SREDA} outputs $\hat\vx$ such that
    $\BE\norm{\nabla \Phi(\hat\vx)} \leq 1504\eps$
    with $\fO(\kappa^3\eps^{-3})$ stochastic gradient evaluations.
\end{thm}

We should point out the complexity shown in Theorem \ref{thm:main} gives optimal dependency on $\eps$.
We consider the special case of minimax problem whose objective function has the form 
\begin{align*}
    f(\vx, \vy) = g(\vx) + h(\vy)
\end{align*}
where $g$ is possibly nonconvex and $h$ is strongly-concave, which leads to minimizing on $\vx$ and maximizing on $\vy$ are independent. 

Consequently, finding $\fO(\eps)$-stationary point of the corresponding $\Phi(\vx)$ can be reduced to finding $\fO(\eps)$-stationary point of nonconvex function $g(\vx)$, which is based on the stochastic first order-oracle $\nabla_\vx F(\vx,\vy;\xi)=\nabla g(\vx;\xi)$ (this equality holds for any $\vy$ since $\vx$ and $\vy$ are independent). Hence, the analysis of stochastic nonconvex minimization problem~\cite{ACDFSW2019} based on $\nabla g(\vx;\xi)$ can directly lead to the $\fO(\eps^{-3})$ lower bound for our minimax problem.
We can prove it by constructing the separate function as 
$f(\vx,\vy) = g(\vx) + h(\vy)$
where $g$ is the nonconvex function in \citeauthor{ACDFSW2019}'s ~\cite{ACDFSW2019} lower bound analysis of stochastic nonconvex minimization,
and $h$ is an arbitrary smooth, $\mu$-strongly concave function.
It is obvious that the lower bound complexity of finding an
$\fO(\eps)$-stationary point of $\Phi$ is no smaller than that of finding an
$\fO(\eps)$-stationary point of $g$, which requires at least
$\fO(\eps^{-3})$ stochastic gradient evaluations~\cite{ACDFSW2019}. 

\subsection{Extension to Finite-sum Case}
SREDA also works for nonconvex-strongly-concave minimax optimization in the finite-sum case (\ref{prob:main-finite}) with little modification of Algorithm \ref{alg:SREDA}.
We just need to replace line 5-7 of Algorithm \ref{alg:SREDA} with the full gradients, and use projected SARAH (PSARAH)\footnote{PSARAH extends SARAH~\cite{nguyen2017sarah} to constrained case, which requires $\fO\left((n+\kappa)\log(\kappa/\eps)\right)$ stochastic gradient evaluation to achieve sufficient accuracy for our initialization. Please  see Appendix E.1 for details} to initialization. We present the details in Algorithm \ref{alg:SREDA-finite}.
The algorithm is more efficient than Minimax PPA~\cite{lin2020near} when $n\geq\kappa^2$.
We state the result formally in Theorem \ref{thm:finite}.

\begin{algorithm}[t]
    \caption{SREDA (Finite-sum Case)}\label{alg:SREDA-finite}
	\begin{algorithmic}[1]
	\STATE \textbf{Input} initial point $\vx_0$, learning rates $\eta_k,\lambda>0$, batch sizes $S_1,S_2>0$; periods $q,m>0$; number of initial iterations $K_0$ \\[0.1cm]
    \STATE \label{line:SARAH}$\vy_0=\text{PSARAH}\left(-f(\vx_k,\cdot),~K_0\right)$ \\[0.1cm]
    \STATE \textbf{for}  $k=0,\dots,K-1$ \textbf{do} \\[0.1cm]
    \STATE\quad \textbf{if} $\mod(k, q)=0$ \\[0.1cm]
    \STATE\quad\quad $\vv_k=\nabla_\vx f(\vx_k,\vy_k)$ \\[0.1cm]
    \STATE\quad\quad $\vu_k=\nabla_\vy f(\vx_k,\vy_k)$ \\[0.1cm]
    \STATE\quad\textbf{else} \\[0.1cm]
    \STATE\label{line:f-vv}\quad\quad $\vv_k=\vv'_{k}$ \\[0.1cm]
    \STATE\label{line:f-vu}\quad\quad $\vu_k=\vu'_{k}$ \\[0.1cm]
    \STATE\quad\textbf{end if}\\[0.1cm]
    \STATE\quad $\vx_{k+1}=\vx_k - \eta_k \vv_k$ \\[0.1cm]
    \STATE\label{line:-finner}\quad
    \STATE \textbf{end for} \\[0.1cm]
	\STATE \textbf{Output} $\hat\vx$ chosen uniformly at random from $\{\vx_i\}_{i=0}^{K-1}$
	\end{algorithmic}
\end{algorithm}

\begin{thm}\label{thm:finite}
Suppose Assumption \ref{asm:bound}-\ref{asm:f} hold. 
In the finite-sum case with $n\geq\kappa^2$, we set the parameters% ($S_1$ is unnecessary)
\begin{align*}
& \zeta = \kappa^{-2}\eps^2,~
\eta_k=\min\left(\dfrac{\eps}{5\kappa\ell\norm{\vv_k}}, \dfrac{1}{10\kappa\ell}\right),~  
\lambda=\dfrac{2}{7\ell},~
q = \lceil\kappa^{-1}n^{1/2}\rceil,~\\[0.1cm]
& S_2 = \left\lceil\frac{3687}{76}\kappa q\right\rceil,~ 
 K = \left\lceil\dfrac{100\kappa\ell \eps^{-2}\Delta_f}{9}\right\rceil,~\text{and}~
m = \left\lceil 1024\kappa \right\rceil. 
\end{align*}
Algorithm \ref{alg:SREDA-finite} outputs $\hat\vx$ such that
$\BE\norm{\nabla \Phi(\hat\vx)} \leq 1504\eps$
with $\fO\left(n\log(\kappa/\eps) + \kappa^2n^{1/2}\eps^{-2}\right)$ stochastic gradient evaluations.

In the case of $n\leq\kappa^2$, we set the parameters %($S_1$ is unnecessary)
\begin{align*}
    & \zeta = \kappa^{-2}\eps^2,~
\eta_k=\min\left(\dfrac{\eps}{5\kappa\ell\norm{\vv_k}}, \dfrac{1}{10\kappa\ell}\right),~
\lambda=\dfrac{1}{8\ell},~ q = 1, \\[0.1cm]
    & S_2 = 1, ~
    K = \left\lceil\dfrac{100\kappa\ell \eps^{-2}\Delta_f}{9}\right\rceil,~\text{and}~
    m = \left\lceil 1024\kappa  \right\rceil.
\end{align*}
Algorithm \ref{alg:SREDA-finite} outputs $\hat\vx$ such that
$\BE\norm{\nabla \Phi(\hat\vx)} \leq 1504\eps$
with $\fO\left((\kappa^2 + \kappa n)\eps^{-2} \right)$ stochastic gradient evaluations.
\end{thm}

\section{Sketch of Proofs}\label{section:proofs}

We present the briefly overview of the proof of Theorem \ref{thm:main}. The details are shown in appendix. Different from \citeauthor{lin2019gradient}'s analysis of SGDA~\cite{lin2019gradient} which directly considered the value of
$\Phi(\vx_k)$ and the distance $\norm{\vy_k-\vy^*(\vx_k)}$, our proof mainly depends on $f(\vx_k, \vy_k)$ and its gradient. We split the change of objective functions after one iteration on $(\vx_k, \vy_k)$ into $A_k$ and $B_k$ as follows
\begin{align}
      f(\vx_{k+1}, \vy_{k+1}) - f(\vx_{k}, \vy_{k}) 
 =   \underbrace{f(\vx_{k+1}, \vy_{k}) - f(\vx_{k}, \vy_{k})}_{A_k} + \underbrace{f(\vx_{k+1}, \vy_{k+1})- f(\vx_{k+1}, \vy_{k})}_{B_k}, \label{ieq:fAB}
\end{align}
where $A_k$ provides the decrease of function value $f$ and $B_k$ can characterize the difference between $f(\vx_{k+1},\vy_{k+1})$ and $\Phi(\vx_{k+1})$.
We can show that
$\BE[A_k] \leq - \fO(\kappa^{-1}\eps)$ and $\BE[B_k] \leq \fO(\kappa^{-1}\eps^2/\ell)$.
By taking the average of (\ref{ieq:fAB}) over $k=0,\dots,K$, we obtain 
\begin{align*}
    \frac{1}{K}\sum_{k=0}^{K-1} \BE \norm{\vv_k} \leq  \fO(\eps).
\end{align*}
We can also approximate $\BE\norm{\nabla\Phi(\vx_k)}$ by $\BE\norm{\vv_k}$ with $\fO(\eps)$ estimate error. Then the output $\hat\vx$ of Algorithm \ref{alg:SREDA} satisfies $\BE\norm{\nabla\Phi(\vx_k)} \leq \fO(\eps)$. Based on the discussion in Section
\ref{subsection:complexity}, the number of stochastic gradient evaluation is $\fO(\kappa^3\eps^{-3})$.  We can also use similar idea to prove Theorem \ref{thm:finite}.

\section{Numerical Experiments}\label{section:experiment}

We conduct the experiments by using distributionally robust optimization with nonconvex regularized logistic loss~\cite{duchi2019variance,yan2019stochastic,kohler2017sub,antoniadis2011penalized}. Given dataset $\left\{(\va_i, b_i)\right\}_{i=1}^n$ where $\va_i\in\BR^d$ is the feature of $i$-th sample and $b_i\in\{1,-1\}$ the corresponding label, the minimax formulation is:
\begin{align*}
    \min_{\vx\in\BR^d} \max_{\vy\in{\fY}} f(\vx,\vy)\triangleq \frac{1}{n}\sum_{i=1}^n \Big(y_i l_i(\vx)-V(\vy)+g(\vx)\Big),
\end{align*}
$l_i(\vx) = \log(1+\exp(-b_i \va_i^\top \vx))$,
$g$ is the nonconvex regularizer~\cite{antoniadis2011penalized}:
\begin{align*}
    g(\vx) = \lambda_2\sum_{i=1}^d \frac{\alpha x_i^2}{1+\alpha x_i^2},
\end{align*}
$V(\vy)=\frac{1}{2}\lambda_1\norm{n\vy-\bf 1}^2$ and ${\mathcal Y}=\{\vy\in{\mathbb R}^n: 0 \leq y_i \leq 1, \sum_{i=1}^n y_i=1\}$ is a simplex. 
Following \citeauthor{yan2019stochastic}~\cite{yan2019stochastic}, \citeauthor{kohler2017sub}~\cite{kohler2017sub}'s settings, we let $\lambda_1=1/n^2$, $\lambda_2 = 10^{-3}$ and $\alpha = 10$ for experiments. 

We evaluate compared the performance of SREDA with baseline algorithms GDAmax, GDA, SGDA~\cite{lin2019gradient} and Minimax PPA~\cite{lin2020near} on six real-world data sets ``a9a'', ``w8a'', ``gisette'',  ``mushrooms'', ``sido0'' and ``rcv1'', whose details are listed in Table \ref{table:datasets}.  The dataset ``sido0'' comes from Causality Workbench\footnote{\url{https://www.causality.inf.ethz.ch/challenge.php?page=datasets}} and the others can be downloaded from LIBSVM repository\footnote{\url{https://www.csie.ntu.edu.tw/~cjlin/libsvmtools/datasets/}}. 
Our experiments are conducted on a workstation with Intel Xeon Gold 5120 CPU and 256GB memory. We use MATLAB 2018a to run the code and the operating system is Ubuntu 18.04.4 LTS.

The parameters of the algorithms are chosen as follows: The stepsizes of all algorithms are tuned from $\{10^{-3},10^{-2},10^{-1},1\}$ and we keep the stepsize ratio is $\{10, 10^2, 10^3\}$. For stochastic algorithms SGDA and SREDA, the mini-batch size is set with $\{10, 100, 200\}$. For SREDA, we use the finite-sum version (Algorithm \ref{alg:SREDA-finite} with the first case of Theorem \ref{thm:finite}) and let $q=m=\lceil n/S_2\rceil$ heuristically. The initialization of SREDA is based on PSARAH with $K_0=5$, $b=1$ and $m=20$. For Minimax PPA, we tune the proximal parameter from $\{1, 10, 100\}$ and momentum parameter from $\{0.2, 0.5, 0.7\}$. Each inner loop of Minimax PPA has five times Maximin-AG2 which contains five AGD iterations. 
The results are shown in Figure \ref{fig:exp}. It is clear that SREDA converges  faster than the baseline algorithms. 

\begin{table}[ht]
\newcommand{\tabincell}[2]{\begin{tabular}{@{}#1@{}}#2\end{tabular}}
\renewcommand{\arraystretch}{1.2}
    \centering
    \begin{tabular}{|c|c|c|}
        \hline
        datasets & $n$ & $d$  \\\hline
        a9a & 32,561 & ~~~~~123  \\
        w8a & 49,749 &  ~~~~~300 \\
        gisette & ~~6,000 & ~~5,000 \\
        mushrooms & ~~8,124 & ~~~~~112 \\
        sido0 & 12,678 & ~~4,932 \\
        rcv1 & 20,242 & 47,236 \\\hline
    \end{tabular}
    \renewcommand{\arraystretch}{1}\vskip0.2cm
    \caption{Summary of datasets used in our experiments}
    \label{table:datasets}
\end{table}

\begin{figure}[ht]
\centering
    \begin{tabular}{ccc}
        \includegraphics[scale=0.32]{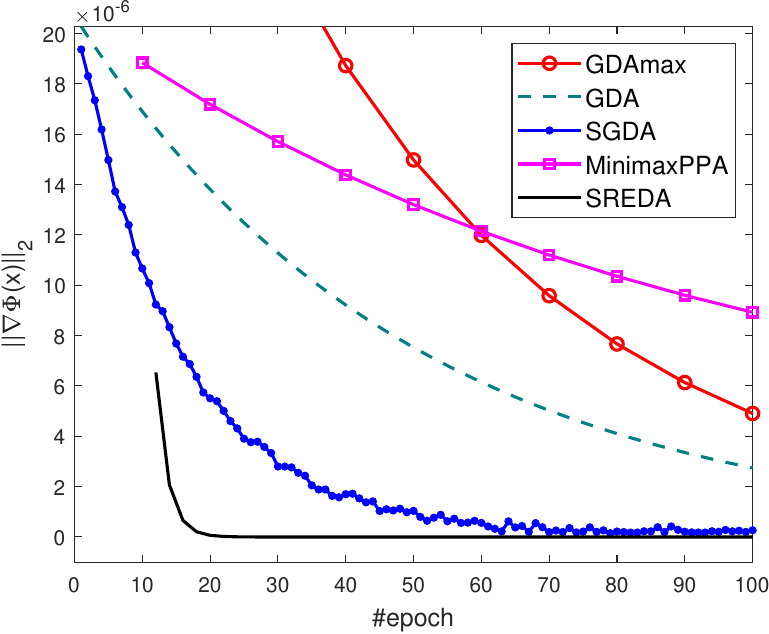} &
        \includegraphics[scale=0.32]{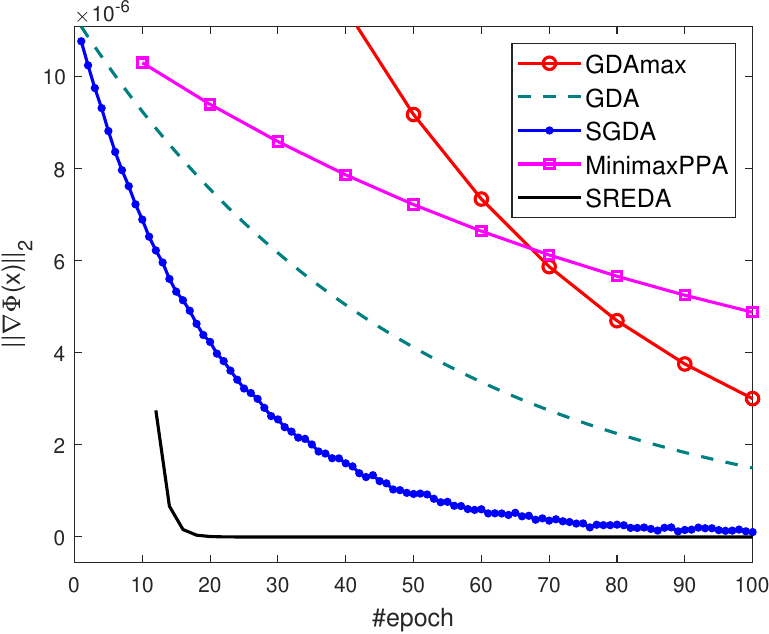} &
        \includegraphics[scale=0.32]{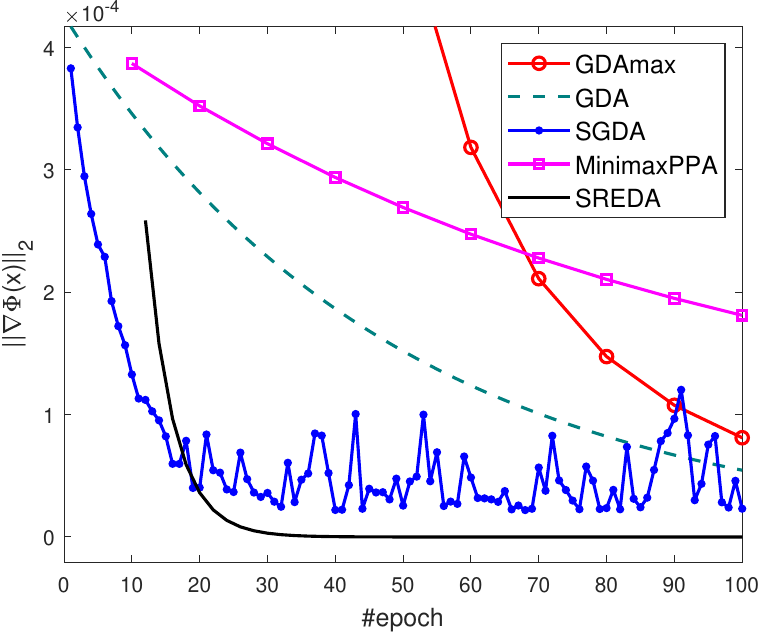} \\
        (a) a9a & (b) w8a & (c) gisette \\[0.15cm]
        \includegraphics[scale=0.32]{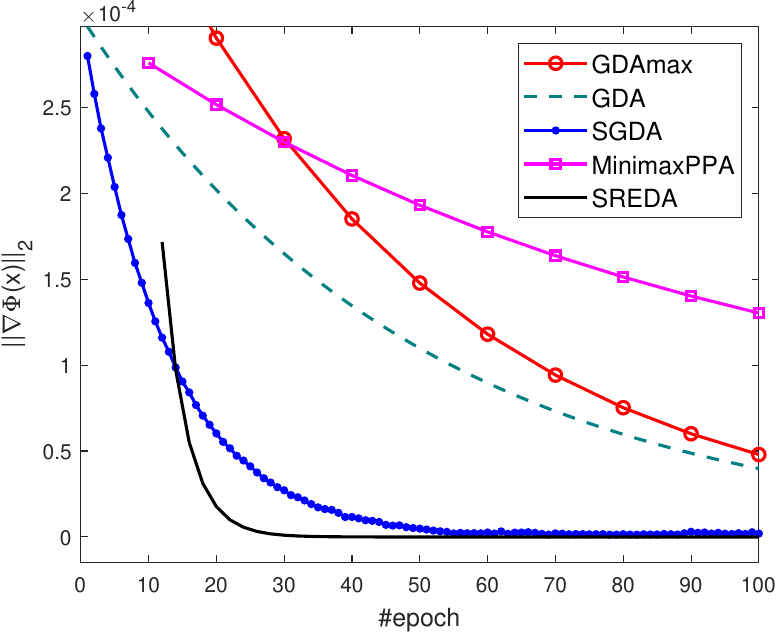} &
        \includegraphics[scale=0.32]{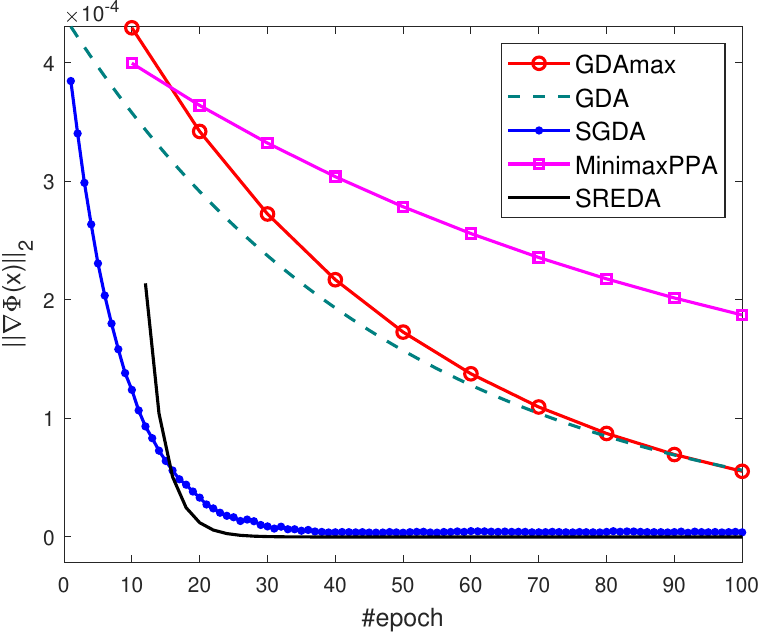} &
        \includegraphics[scale=0.32]{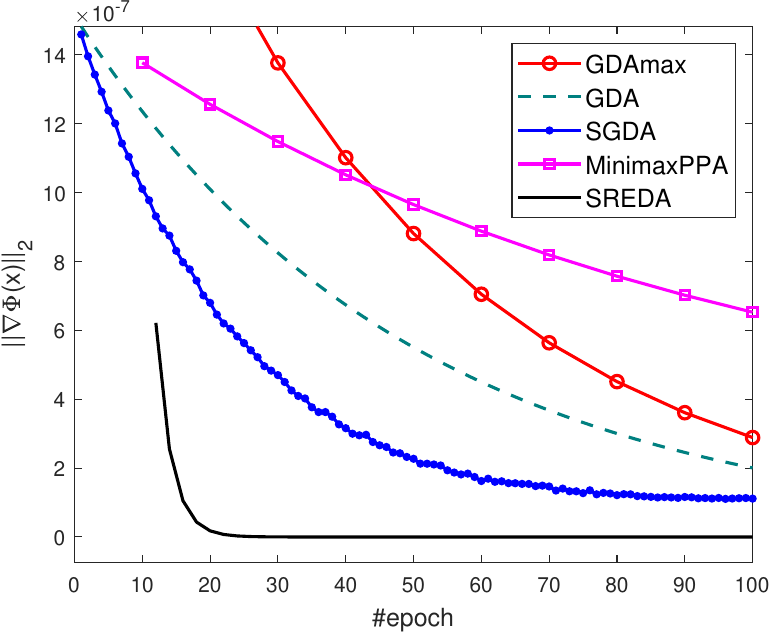} \\ 
        (d) mushrooms & (e) sido0 & (f) rcv1 \\[0.15cm]
    \end{tabular}
    \caption{We demonstrate $\norm{\nabla\Phi(\vx)}$ vs. the number of epochs for DRO model on real-world datasets ``a9a'', ``w8a'', ``gisette'', ``mushrooms'', ``sido0'' and ``rcv1'' with SREDA and baseline algorithms.}\label{fig:exp}
\end{figure}

\section{Conclusion}\label{section:conclusions}

In this paper, we studied stochastic nonconvex-strongly-concave minimax problems.
We proposed a novel algorithm called Stochastic Recursive gradiEnt Descent Ascent (SREDA).
The algorithm employs variance reduction to solve minimax problems.
Based on the appropriate choice of the parameters,
we prove SREDA finds an $\fO(\eps)$-stationary point of $\Phi$ with a
stochastic gradient complexity of ${\mathcal
  O}(\kappa^3\varepsilon^{-3})$. This result is better than
state-of-the-art algorithms and optimal in its dependency on $\eps$.
We can also apply SREDA to the finite-sum case, and show that it performs well when $n$ is larger than $\kappa^2$.

There are still some open problems left.
The complexity of SREDA is optimal with respect to $\eps$, but weather
it is optimal with respect to $\kappa$ is unknown.
It is also possible to employ SREDA to reduce the complexity of
stochastic nonconvex-concave minimax problems without the strongly-concave assumption.

\section*{Broader Impact}

This paper studied the theory of stochastic minimax optimization. The proposed method SREDA is the first stochastic algorithm which attains the optimal dependency on $\eps$. This observation help us to understand the minimax optimization without convex-concave assumption. It is interesting to apply SREDA to more machine learning applications in future.

\begin{ack}
The authors would like to thank Min Tao and Jiahao Xie to point out that the first version of this paper on arXiv has a mistake in the original proof of Theorem \ref{thm:main}.
This work is supported by GRF 16201320 and the project of Shenzhen Research Institute of Big Data (named ``Automated Machine
Learning'').
\end{ack}

{\bibliographystyle{plainnat}
\bibliography{reference}}

\clearpage
\newpage

\appendix

\section*{Supplementary Materials}

This supplementary materials are organized as follows. Appendix \ref{app:tool} provide several technique lemmas for later analysis. Appendix \ref{app:inner} give some properties for our concave maximizer. Then, Appendix \ref{app:PiSARAH} proposes projected inexact SARAH (PiSARAH), which generalizes SARAH to constrained optimization and we use it for the initialization of SREDA. Appendix \ref{app:main} presents the proof of our main results Theorem \ref{thm:main} and we extend it to prove finite-sum case Theorem \ref{thm:finite} in Appendix \ref{app:finite}. %Finally, Appendix \ref{app:exp} gives some details of our experimental settings.

\section{Technical Tools}\label{app:tool}

We first present some useful inequalities in convex optimization, martingale variance bound and gradient mapping.
\begin{lem}[{\cite[Theorem 2.1.5 and 2.1.12]{nesterov2018lectures}}]\label{lem:convex}
    Suppose $g(\cdot)$ is $\mu$-strongly convex and has $\ell$-Lipschitz gradient.
    Let $\vw^*$ be the minimizer of $g$.
    Then for any $\vw$ and $\vw'$,
    we have the following inequalities
    \begin{align}
         & \inner{\nabla g(\vw)-\nabla g(\vw')}{\vw-\vw'} \geq \frac{1}{\ell}\norm{\nabla g(\vw) - \nabla g(\vw')}^2, \label{ieq:grad-smooth}\\
         & \inner{\nabla g(\vw) - \nabla g(\vw')}{\vw-\vw'} \geq  \frac{\mu \ell}{\mu+\ell} \norm{\vw-\vw'}^2 + \frac{1}{\mu+\ell}\norm{\nabla g(\vw) - \nabla g(\vw')}^2, \label{ieq:grad-smooth0} % \\
         %& g(\vw)+\inner{\nabla g(\vw)}{\vw'-\vw} + \frac{1}{2\ell}\norm{\nabla g(\vw)-\nabla g(\vw')}^2 \leq g(\vw') \label{ieq:grad-smooth1}.
    \end{align}
\end{lem}

\begin{lem}[{\cite[Lemma 1]{fang2018spider}}]\label{lem:martingale}
    Let $\fV_k$ be estimator of $\fB(\vz_k)$ as
    \begin{align*}
        \fV_k = \fB_{\fS_*}(\vz_k) - \fB_{\fS_*}(\vz_{k-1}) + \fV_{k-1},
    \end{align*}
    where $\fB_{\fS_*}=\frac{1}{|\fS_*|}\sum_{\fB_i\in\fS_*}\fB_i$ satisfies
    \begin{align*}
        \BE\left[\fB_i(\vz_k)-\fB_i(\vz_{k-1}) \mid \vz_0,\dots,\vz_{k-1}\right]
        =  \BE\left[\fV_k-\fV_{k-1} \mid \vz_0,\dots,\vz_{k-1}\right],
    \end{align*}
    and $\fB_i$ is $L$-Lipschitz continuous for any $\fB_i\in\fS_*$.
    Then for all $k=1,\dots,K$, we have
    \begin{align*}
        \BE\norm{\fV_k-\fB(\vz_k)\mid\vz_0,\dots,\vz_{k-1}}^2
        \leq & \norm{\fV_{k-1}-\fB(\vz_{k-1})}^2 + \frac{\ell^2}{|\fS_*|}\BE\left[\norm{\vz_k - \vz_{k-1}}^2\mid \vz_0,\dots,\vz_{k-1}\right].
    \end{align*}
\end{lem}

\begin{lem}[{\cite[Corollary 2.2.3]{nesterov2018lectures}}]\label{lem:nonexpansive}
    Given convex and compact set $\fC\subseteq\BR^d$ and any $\vw, \vw'\in\BR^d$, we have $\norm{\Pi_\fC(\vw)-\Pi_\fC(\vw')}\leq\norm{\vw-\vw'}$.
\end{lem}

\begin{lem}[{\cite[Corollary 2.2.4]{nesterov2018lectures}}]\label{lem:prox2}
Let $g$ has $\ell$-Lipschitz gradient and $\mu$-strongly convex, and $\fC$ is a convex set.
Denote $\fG_\gamma$ be the gradient mapping such that
\begin{align*}
    \fG_\gamma(\vw) = \frac{\vw-\Pi_\fC (\vw-\gamma\nabla g(\vw))}{\gamma},
\end{align*}
where $\gamma \leq 1/\ell$. Let $\vw^*=\argmin_{\vw\in\fC} g(\vw)$,
then we have
\begin{align*}
    \inner{\fG_\gamma(\vw)}{\vw-\vw^*}^2 \geq \frac{\mu}{2}\norm{\vw-\vw^*} + \frac{1}{2\ell}\norm{\fG_\gamma(\vw)}^2
\end{align*}
\end{lem}

\begin{cor}\label{cor:prox2}
Under assumptions of Lemma \ref{lem:prox2},
we have $\frac{\mu}{2}\norm{\vw-\vw^*} \leq \norm{\fG_\gamma(\vw)}$.
\end{cor}
\begin{proof}
We have
\begin{align*}
       \frac{\mu}{2}\norm{\vw-\vw^*}^2 
    \leq & \inner{\fG_\gamma(\vw)}{\vw-\vw^*} - \frac{1}{2\ell}\norm{\fG_\gamma(\vw)}^2 \\
    \leq & \inner{\fG_\gamma(\vw)}{\vw-\vw^*} \\
    \leq & \norm{\fG_\gamma(\vw)}\norm{\vw-\vw^*},
\end{align*}
where the first inequality use Lemma \ref{lem:prox2} and the last one is based on Cauchy-Schwarz inequality.
Then we obtain the desired result.
\end{proof}

\section{Some Results of Concave Maximizer}\label{app:inner}

In this section, we present some results of concave maximizer Algorithm \ref{alg:inner}. The analysis of SREDA is based on the following two auxiliary quantities:
\begin{align*}
    \Delta_k = \BE\left[\norm{\vv_k-\nabla_\vx f(\vx_k,\vy_k)}^2+\norm{\vu_k-\nabla_\vy f(\vx_k,\vy_k)}^2\right] \text{~~and~~}
    \delta_k=\BE\norm{\fG_{\lambda,\vy}(\vx_k,\vy_k)}^2. 
\end{align*}
The main target is to prove both $\Delta_k$ and $\delta_k$ can be bounded by $\fO(\kappa^{-2}\eps^2)$.

Using the notations of Algorithm \ref{alg:SREDA} and \ref{alg:inner},
we denote $\ty_k^*=\argmin_{\vy\in\fY} g_k(\vy)$ and $\hu_{k,t}=-\tu_{k,t}$.
It is obvious that $g_k(\cdot)$ is $\mu$-strongly convex and has $\ell$-Lipschitz gradient.
The update rule of $\vx_k$ in Algorithm \ref{alg:SREDA} means for any $k\geq 0$, we have
\begin{align*}
    \norm{\vx_{k+1}-\vx_k}^2 \leq \eps_\vx^2, 
\end{align*}
where $\eps_\vx^2$ is defined as $\frac{1}{25}\kappa^{-2}\eps^2$.
We also denote the gradient mapping with respect to $\vy$ as
\begin{align*}
    \tilde\fG_{\lambda,k}(\vy)\triangleq \frac{\vy-\Pi_{\fY }(\vy-\lambda\nabla g_k(\vy))}{\lambda} = \fG_{\lambda}(\vx_{k+1},\vy).
\end{align*}

We first introduce some lemmas for our iteration and gradient mapping.
\begin{lem}[Lemma 1 of \cite{li2018simple}]\label{lem:prox-compositive}
Let $\vy^+ := \Pi_{\fY}(\vy-\lambda \vu)$, then for all $\vz$, we have
\begin{align*}
g_k(\vy^+)
\leq g_k(\vz) + \inner{\nabla g_k(\vy)-\vu}{\vy^+-\vz}  
    -\frac{\inner{\vy^+-\vy}{\vy^+-\vz}}{\lambda}
    +\frac{\ell}{2}\norm{\vy^+-\vy}^2+\frac{\ell}{2}\norm{\vz-\vy}^2.
\end{align*}
\end{lem}

\begin{lem}[Lemma 2 of \cite{li2018simple}]\label{lem:prox-compositive2} Let $\vy^+ := \Pi_{\fY}(\vy-\lambda \vu)$ and $\overline{\vy_{k,t}} := \Pi_{\fY}\big(\vy-\lambda \nabla g_k(\vy)\big)$, then we have
$\inner{\nabla g_k(\vy)-\vu}{\vy^+-\overline{\vy_{k,t}}} \leq \lambda\norm{\nabla g_k(\vy)-\vu}^2$.
\end{lem}

\begin{lem}\label{lem:diff-grad-mapping}
For Algorithm \ref{alg:SREDA} and any $\vy$, we have
$\norm{\fG_{\lambda,\vy} (\vx_{k+1},\vy)-\fG_{\lambda,\vy} (\vx_{k},\vy)}^2 \leq \ell^2\eps_\vx^2$.
\end{lem}
\begin{proof}
Using Lemma \ref{lem:nonexpansive} and smoothness of $f$, we have
\begin{align*}
  & \norm{\fG_{\lambda,\vy} (\vx_{k+1},\vy_{k})-\fG_{\lambda,\vy} (\vx_{k},\vy_{k})}^2 \\
= & \norm{\frac{\vy_k-\Pi_{\fY}(\vy_k-\lambda\nabla _\vy f(\vx_{k+1},\vy_k))}{\lambda} - \frac{\vy_k-\Pi_{\fY}(\vy_k-\lambda\nabla_\vy f(\vx_{k},\vy_k))}{\lambda}}^2 \\
\leq & \frac{1}{\lambda^2}\norm{\Pi_{\fY}(\vy_k-\lambda\nabla_\vy f(\vx_{k+1},\vy_k)) - \Pi_{\fY}(\vy_k-\lambda\nabla_\vy f(\vx_{k},\vy_k))}^2 \\
\leq & \norm{\nabla_\vy f(\vx_{k+1},\vy_k)-\nabla_\vy f(\vx_{k},\vy_k)}^2 \leq \ell^2\eps_\vx^2.
\end{align*}
\end{proof}

\begin{lem}\label{lem:PL}
Let $\vy^+=\Pi_\fY(\vy-\lambda\vu)$ and $\lambda<1/\ell$, then we have
\begin{align*}
   g_k(\vy^+)-g_k(\vy^*) 
\leq  \frac{1}{\mu}\left(\|{\tilde\fG_{\lambda,k}(\vy)}\|_2^2 + \norm{\nabla g_k(\vy)-\vu}^2\right)
\end{align*}
for any $\vy^*\in\fY$.
\end{lem}
\begin{proof}
Let $Q(\vz)=f(\vy)+\inner{\vu}{\vz-\vy}+\frac{1}{2\lambda}\norm{\vz-\vy}^2+r(\vz)$. We have $\vy^+=\argmin_\vz Q(\vz)$ and
\begin{align*}
    \inner{\vu + \lambda^{-1}(\vy^+-\vy) + \vxi}{\vy^+-\vy^*} = \inner{\nabla Q(\vy^+)}{\vy^*-\vy^+} \geq 0,
\end{align*}
for any $\vy^*\in\fY$.
Then
\begin{align*}
   & \inner{\nabla g_k(\vy^+)}{\vy^*-\vy^+} \\
=  & \inner{\nabla g_k(\vy^+) - \nabla g_k(\vy)}{\vy^*-\vy^+} + \inner{\nabla g_k(\vy)}{\vy^*-\vy^+} \\
= & \inner{\nabla g_k(\vy^+) - \nabla g_k(\vy)}{\vy^*-\vy^+} + \inner{\nabla g_k(\vy)-\vu+ \lambda^{-1}(\vy-\vy^+)}{\vy^*-\vy^+} \\
= & \inner{\nabla{\tilde g_k}(\vy^+) - \nabla {\tilde g_k}(\vy)}{\vy^*-\vy^+} + \inner{\nabla g_k(\vy)-\vu}{\vy^*-\vy^+} \\
\geq & -\max(\ell,\lambda^{-1})\norm{\vy^+-\vy}\norm{\vy^*-\vy^+} + \inner{\nabla g_k(\vy)-\vu}{\vy^*-\vy^+} \\
\geq & -\max(\ell,\lambda^{-1})\norm{\vy^+-\vy}\norm{\vy^*-\vy^+} - \norm{\nabla g_k(\vy)-\vu}\norm{\vy^*-\vy^+},
\end{align*}
where $\tilde g_k(\vy) = g_k(\vy)-\frac{1}{2\lambda}\norm{\vy}^2$.
The first inequality is due to $\tilde g_k$ is at most $\max(\ell,\lambda^{-1})$-smooth, that is
\begin{align*}
    -\lambda^{-1}\mI \preceq (\mu-\lambda^{-1})\mI 
    \preceq \nabla^2 \tilde g_k(\vy)
    \preceq (\ell-\lambda^{-1})\mI \preceq \ell\mI.
\end{align*}
Consequently, we have
\begin{align*}
 & -\max(\ell,\lambda^{-1})\norm{\vy^+-\vy}\norm{\vy^*-\vy^+} \\
\leq & \inner{\nabla g_k(\vy^+)}{\vy^*-\vy^+} + \norm{\nabla g_k(\vy)-\vu}\norm{\vy^*-\vy^+}    \\
\leq & g_k(\vy^*)-g_k(\vy^+) -\frac{\mu}{2}\norm{\vy^*-\vy^+}^2 + \norm{\nabla g_k(\vy)-\vu}\norm{\vy^*-\vy^+} 
\end{align*}
which implies
\begin{align*}
 & g_k(\vy^*)-g_k(\vy^+) \\
\geq & \frac{\mu}{2}\norm{\vy^*-\vy^+}^2  -\max(\ell,\lambda^{-1})\norm{\vy^+-\vy}\norm{\vy^*-\vy^+} - \norm{\nabla g_k(\vy)-\vu}\norm{\vy^*-\vy^+}    \\
\geq & \inf_\vy\left\{\frac{\mu}{2}\norm{\vz-\vy^+}^2  -\left(\max(\ell,\lambda^{-1})\norm{\vy^+-\vy}
+\norm{\nabla g_k(\vy)-\vu}\right)\norm{\vz-\vy^+}\right\} \\
= & -\frac{1}{2\mu}\left(\max(\ell,\lambda^{-1})\norm{\vy^+-\vy}
+\norm{\nabla g_k(\vy)-\vu}\right)^2 \\
\geq & -\frac{1}{\mu}\left(\max(\ell,\lambda^{-1})^2\norm{\vy^+-\vy}^2 + \norm{\nabla g_k(\vy)-\vu}^2\right).
\end{align*}
Considering that $\eta\leq 1/\ell$, we can obtain the desired result by rearranging above inequality.
\end{proof}

We can now present the key lemma for the concave maximizer, which upper bounds the magnitude of gradient mapping after one epoch iterations on $\vy$.
\begin{lem}\label{lem:SR}
For Algorithm \ref{alg:inner} with $\lambda=\frac{1}{8\ell}$, we have 
\begin{align*}
  & \BE\|\tilde\fG_{\lambda,k}(\ty_{k,s_k})\|_2^2  \\
  %{\color{blue}{~~//~~ \tilde\fG_{\lambda,k}(\ty_{k,s_k})=\fG_{\lambda,\vy}(\vx_{k+1},\vy_{k+1}),%\tilde\fG_{\lambda,k}(\ty_{k,0})=\fG_{\lambda}(\vx_{k+1},\vy_{k})}} \\
\leq & \frac{64\ell}{m\mu}\BE\|\tilde\fG_{\lambda,k}(\ty_{k,0})\|_2^2 + \left(\frac{64\ell}{m\mu}+8\right)\BE\norm{\nabla g_k(\ty_{k,0})-\tu_{k,0}}^2 + 8\ell^2\BE\norm{\ty_{k,1}-\ty_{k,0}}^2.
\end{align*}
\end{lem}
\begin{proof}
We define $\overline{\vy_{k,t}} := \Pi_{\fY}\big(\ty_{k,t-1}-\lambda \nabla g_k(\ty_{k,t-1})\big)$.
The procedure of Algorithm \ref{alg:inner} means $\ty_{k,t} := \Pi_{\fY}(\ty_{k,t-1}-\lambda \hu_{k,t-1})$.
Using Lemma \ref{lem:prox-compositive} by letting $\vy^+=\ty_{k,t}, \vy=\ty_{k,t-1}, \vu=\hu_{k,t-1}$ and $\vz=\overline{\vy_{k,t}}$ , we have
\begin{align}
\begin{split}
g_k(\ty_{k,t})\leq & g_k(\overline{\vy_{k,t}}) + \inner{\nabla g_k(\ty_{k,t-1})-\tu_{k,t-1}}{\ty_{k,t}-\overline{\vy_{k,t}}} \\
& -\frac{1}{\lambda}\inner{\ty_{k,t}-\ty_{k,t-1}}{\ty_{k,t}-\overline{\vy_{k,t}}} +\frac{\ell}{2}\norm{\ty_{k,t}-\ty_{k,t-1}}^2+\frac{\ell}{2}\norm{\overline{\vy_{k,t}}-\ty_{k,t-1}}^2.
\end{split}\label{ieq:prox1}
\end{align}
Using Lemma \ref{lem:prox-compositive} again by letting $\vy^+=\overline{\vy_{k,t}}, \vy=\ty_{k,t-1}, \vu=\nabla g_k(\ty_{k,t-1})$ and $\vz=\vy=\ty_{k,t-1}$, we have
\begin{align}
\begin{split}
g_k(\overline{\vy_{k,t}})
\leq & g_k(\ty_{k,t-1}) -\frac{1}{\lambda}\inner{\overline{\vy_{k,t}}-\ty_{k,t-1}}{\overline{\vy_{k,t}}-\ty_{k,t-1}}
    +\frac{\ell}{2}\norm{\overline{\vy_{k,t}}-\ty_{k,t-1}}^2 \\
= & g_k(\vy_{t-1}) - \left(\frac{1}{\lambda}-\frac{\ell}{2}\right)\norm{\overline{\vy_{k,t}}-\ty_{k,t-1}}^2.
\end{split}\label{ieq:prox2}
\end{align}
Sum over inequalities (\ref{ieq:prox1}) and (\ref{ieq:prox2}), we have
\begin{align*}
     & g_k(\ty_{k,t}) \\
\leq & g_k(\ty_{k,t-1}) + \frac{\ell}{2}\norm{\ty_{k,t}-\ty_{k,t-1}}^2 -\left(\frac{1}{\lambda}-\ell\right)\norm{\overline{\vy_{k,t}}-\ty_{k,t-1}}^2 \\
& + \inner{\nabla g_k(\ty_{k,t-1})-\hu_{k,t-1}}{\ty_{k,t}-\overline{\vy_{k,t}}} -\frac{1}{\lambda}\inner{\ty_{k,t}-\ty_{k,t-1}}{\ty_{k,t}-\overline{\vy_{k,t}}}^2 \\
=& g_k(\ty_{k,t-1}) + \frac{\ell}{2}\norm{\ty_{k,t}-\ty_{k,t-1}}^2 -\left(\frac{1}{\lambda}-\ell\right)\norm{\overline{\vy_{k,t}}-\ty_{k,t-1}}^2 \\
 &  + \inner{\nabla g_k(\ty_{k,t-1})-\hu_{k,t-1}}{\ty_{k,t}-\overline{\vy_{k,t}}} \\ 
 & -\frac{1}{2\lambda}\left(\norm{\ty_{k,t}-\ty_{k,t-1}}^2+\norm{\ty_{k,t}-\overline{\vy_{k,t}}}^2-\norm{\overline{\vy_{k,t}}-\ty_{k,t-1}}^2\right) \\
=& g_k(\ty_{k,t-1}) -\left(\frac{1}{2\lambda}-\frac{\ell}{2}\right)\norm{\ty_{k,t}-\ty_{k,t-1}}^2 -\left(\frac{1}{2\lambda}-\ell\right)\norm{\overline{\vy_{k,t}}-\ty_{k,t-1}}^2 \\
 &  + \inner{\nabla g_k(\ty_{k,t-1})-\hu_{k,t-1}}{\ty_{k,t}-\overline{\vy_{k,t}}}  -\frac{1}{2\lambda}\norm{\ty_{k,t}-\overline{\vy_{k,t}}}^2 \\
\leq & g_k(\ty_{k,t-1}) -\left(\frac{1}{2\lambda}-\frac{\ell}{2}\right)\norm{\ty_{k,t}-\ty_{k,t-1}}^2 -\left(\frac{1}{2\lambda}-\ell\right)\norm{\overline{\vy_{k,t}}-\ty_{k,t-1}}^2 \\
& + \inner{\nabla g_k(\ty_{k,t-1})-\hu_{k,t-1}}{\ty_{k,t}-\overline{\vy_{k,t}}} -\frac{1}{8\lambda}\norm{\ty_{k,t}-\ty_{k,t-1}}^2+\frac{1}{6\lambda}\norm{\overline{\vy_{k,t}}-\ty_{k,t-1}}^2 \\
=& g_k(\ty_{k,t-1}) -\left(\frac{5}{8\lambda}-\frac{\ell}{2}\right)\norm{\ty_{k,t}-\ty_{k,t-1}}^2 -\left(\frac{1}{3\lambda}-\ell\right)\norm{\overline{\vy_{k,t}}-\ty_{k,t-1}}^2 \\
& + \inner{\nabla g_k(\ty_{k,t-1})-\hu_{k,t-1}}{\ty_{k,t}-\overline{\vy_{k,t}}} \\
\leq & g_k(\ty_{k,t-1}) -\left(\frac{5}{8\lambda}-\frac{\ell}{2}\right)\norm{\ty_{k,t}-\ty_{k,t-1}}^2 -\left(\frac{1}{3\lambda}-\ell\right)\norm{\overline{\vy_{k,t}}-\ty_{k,t-1}}^2 \\
& + \lambda\norm{\nabla g_k(\ty_{k,t-1})-\hu_{k,t-1}}^2,  
\end{align*}
where the second inequality uses Young's inequality as follows 
\begin{align*}
 \norm{\ty_{k,t}-\ty_{k,t-1}}^2\leq \left(1+\alpha^{-1}\right)\norm{\overline{\vy_{k,t}}-\ty_{k,t-1}}^2 +(1+\alpha)\norm{\ty_{k,t}-\overline{\vy_{k,t}}}^2  \text{~~with~} \alpha=3,
\end{align*}
and the last inequality holds due to Lemma \ref{lem:prox-compositive2}. Take the expectation on above result, we have
\begin{align}
\begin{split}
    & \BE [g_k(\ty_{k,t+1})] \\
\le & \BE \Bigg[g_k(\ty_{k,t}) - \left(\frac{1}{2\lambda} - \frac{\ell}{2}\right)\norm{\ty_{k,t+1}-\ty_{k,t}}^2 - \left(\frac{1}{3\lambda}-\ell\right)\norm{\overline{\vy_{k,t+1}} - \ty_{k,t}}^2 \\
    & \quad\quad + \lambda\norm{\nabla g_k(\ty_{k,t})-\hu_{k,t}}^2 \Bigg] \\
\le & \BE \Bigg[g_k(\ty_{k,t}) - \left(\frac{1}{2\lambda} - \frac{\ell}{2}\right)\norm{\ty_{k,t+1}-\ty_{k,t}}^2 - \left(\frac{1}{3\lambda}-\ell\right)\norm{\overline{\vy_{k,t+1}} - \ty_{k,t}}^2 \\
& \quad\quad+ \lambda\left(\norm{\nabla g_k(\ty_{k,0})-\tu_{k,0}}^2+\frac{\ell^2}{S_2}\sum_{i=0}^{t-1} \norm{\ty_{k,i+1}-\ty_{k,i}}^2\right) \Bigg],
\end{split}\label{ieq:bound-exp-phi}
\end{align}
where the second inequality is based on Lemma \ref{lem:martingale}. 
Summing over (\ref{ieq:bound-exp-phi}) with $t$ from $1$ to $m$ and relax the upper bound of $i$ to $m$, we obtain
\begin{align*}
    & \BE [g_k(\ty_{k,m})] \\
\leq & \BE [g_k(\ty_{k,1})] - \sum_{i=1}^{m}\left(\frac{1}{2\lambda} - \frac{\ell}{2}
    - \frac{\lambda \ell^2m}{S_2}\right) \BE\norm{\ty_{k,i+1}-\ty_{k,i}}^2 \\
 & - \left(\frac{1}{3\lambda}-\ell\right)\sum_{i=1}^{m}\BE\norm{\overline{\vy_{k,i+1}} - \ty_{k,i}}^2 
  + m\lambda\norm{\nabla g_k(\ty_{k,0})-\hu_{k,0}}^2 + \frac{\lambda \ell^2m}{S_2} \BE\norm{\ty_{k,1}-\ty_{k,0}}^2 
\end{align*}
Consider that $\lambda = \frac{1}{8\ell}$, we further obtain that
{\begin{align*}
    & \BE[g_k(\ty_{k,m})] \\
\leq & \BE[g_k(\ty_{k,1})] - 3\ell\sum_{i=1}^{m}\BE\norm{\ty_{k,i+1}-\ty_{k,i}}^2 - \ell\lambda^2\sum_{i=1}^{m}\BE\|\tilde\fG_{\lambda,k}(\ty_{k,i})\|_2^2 \\
& + m\lambda\norm{\nabla g_k(\ty_{k,1})-\hu_{k,0}}^2 + \frac{\lambda \ell^2m}{S_2} \BE\norm{\ty_{k,1}-\ty_{k,0}}^2 \\
\leq & \BE[g_k(\ty_{k,1})] - \ell\lambda^2\sum_{i=1}^{m}\BE\|\tilde\fG_{\lambda,k}(\ty_{k,i})\|_2^2 + m\lambda\norm{\nabla g_k(\ty_{k,1})-\hu_{k,0}}^2 + \frac{\lambda \ell^2m}{S_2} \BE\norm{\ty_{k,1}-\ty_{k,0}}^2
\end{align*}}
Let $\ty^*_k=\argmin_{\vy\in\fY} g_k(\vy)$, then 
the above inequality further implies that
\begin{align*}
    \sum_{i=1}^{m}\BE\|\tilde\fG_{\lambda,k}(\ty_{k,i})\|_2^2  
\leq 64\ell\left(\BE[g_k(\ty_{k,1}) - g_k(\ty^*_{k})]\right) + 8m\norm{\nabla g_k(\ty_{k,0})-\hu_{k,0}}^2.
\end{align*}
Since $\vy_{k+1}=\ty_{k,s_k}$ and $s_k$ is sampled from $\{1,\dots,m\}$, we have
{\small\begin{align*}
  & \BE\|\tilde\fG_{\lambda,k}(\ty_{k,s_k})\|_2^2 \\
= & \frac{1}{m}\sum_{i=1}^{m}\BE\|\tilde\fG_{\lambda,k}(\ty_{k,i})\|_2^2 \\
\leq & \frac{64\ell}{m}\left(\BE[g_k(\ty_{k,1}) - g_k(\ty_k^*)]\right) + 8\BE\norm{\nabla g_k(\ty_{k,0})-\tu_{k,0}}^2 + \frac{8\ell^2}{S_2} \BE\norm{\ty_{k,1}-\ty_{k,0}}^2 \\
\leq & \frac{64\ell}{m\mu}\left(\BE\|\tilde\fG_{\lambda,k}(\ty_{k,0})\|_2^2+\norm{\nabla g_k(\ty_{k,0})-\tu_{k,0}}^2\right) 
 + 8\BE\norm{\nabla g_k(\ty_{k,0})-\tu_{k,0}}^2 + 8\ell^2\BE\norm{\ty_{k,1}-\ty_{k,0}}^2 \\
= & \frac{64\ell}{m\mu}\BE\|\tilde\fG_{\lambda,k}(\ty_{k,0})\|_2^2 + \left(\frac{64\ell}{m\mu}+8\right)\BE\norm{\nabla g_k(\ty_{k,0})-\tu_{k,0}}^2 + 8\ell^2\BE\norm{\ty_{k,1}-\ty_{k,0}}^2,
\end{align*}}
where the first inequality is based on Lemma \ref{lem:PL} and the second one is due to assumption $S_2\geq m$.
\end{proof}

We bound the progress of $\ty_{k,t}$ by the following lemmas.
\begin{lem}\label{lem:progress-decay}
For Algorithm \ref{alg:inner} with $\lambda\leq 2/(\mu+\ell)$, we have
\begin{align*}
    \BE\norm{\ty_{k,t+1}-\ty_{k,t}}^2
    \leq\left(1 -\frac{2\lambda\mu \ell}{\mu+\ell}\right)\norm{\ty_{k,t}-\ty_{k,t-1}}^2
    \quad \text{for any~~} t \geq 1.
\end{align*}
\end{lem}
\begin{proof}
Using the notations of Algorithm \ref{alg:inner}, we define $g_{k,t}(\vy) = -\frac{1}{S_2}\sum_{i=1}^{S_2} \nabla_\vy F(\vx_{k+1},\vy;\vxi_{t,i})$.
Then, we have
\begin{align*}
  & \BE\norm{\ty_{k,t+1}-\ty_{k,t}}^2 \\
= & \BE\norm{\Pi_{\fY}(\ty_{k,t}-\lambda \hu_{k,t}) - \Pi_{\fY}(\ty_{k,t-1}-\lambda \hu_{k,t-1})}^2 \\
\leq & \BE\norm{(\ty_{k,t}-\lambda \hu_{k,t}) - (\ty_{k,t-1}-\lambda \hu_{k,t-1})}^2 \\
= & \norm{\ty_{k,t}-\ty_{k,t-1}}^2 -2\lambda\BE\inner{\ty_{k,t}-\ty_{k,t-1}}{\hu_{k,t}-\hu_{k,t-1}}+ \lambda^2\BE\norm{\hu_{k,t} - \hu_{k,t-1}}^2 \\
= & \norm{\ty_{k,t}-\ty_{k,t-1}}^2 -2\lambda\BE\inner{\ty_{k,t}-\ty_{k,t-1}}{\nabla g_{k,t-1}(\ty_{k,t})-\nabla g_{k,t-1}(\ty_{k,t-1})} \\
& + \lambda^2\BE\norm{\nabla g_{k,t-1}(\ty_{k,t})-\nabla g_{k,t-1}(\ty_{k,t-1})}^2 \\
\leq & \norm{\ty_{k,t}-\ty_{k,t-1}}^2 -\frac{2\lambda\mu \ell}{\mu+\ell}\BE\norm{\ty_{k,t}-\ty_{k,t-1}}^2  \\ 
& + \left(\lambda^2-\frac{2\lambda}{\mu+\ell}\right)\BE\norm{\nabla g_{k,t-1}(\ty_{k,t})-\nabla g_{k,t-1}(\ty_{k,t-1})}^2 \\
\leq & \left(1 -\frac{2\lambda\mu \ell}{\mu+\ell}\right)\norm{\ty_{k,t}-\ty_{k,t-1}}^2,
\end{align*}
where the first inequality is based on Lemma \ref{lem:nonexpansive}, the second one comes from inequality (\ref{ieq:grad-smooth0}) of Lemma \ref{lem:convex} and the last one is due to $\lambda<2/(\mu+\ell)$. 
\end{proof}

Note that our algorithm estimate $\fG_{\lambda,\vy}(\vx_{k+1},\ty_{k,0})$ by $\frac{1}{\lambda}(\ty_{k,0}- \ty_{k,1})$, whose norm can be bounded as follows:
\begin{lem}\label{lem:diff-y0-y1}
For Algorithm \ref{alg:inner}, we have
\begin{align*}
    \norm{\frac{\ty_{k,0}-\ty_{k,1}}{\lambda}}^2 
\leq 3(\Delta_{k}+2\ell^2\eps_\vx^2+\delta_k).
\end{align*}
\end{lem}
\begin{proof}
The fact $\norm{\va+\vb+\vc}^2 \leq 3\left(\norm{\va}^2+\norm{\vb}^2+\norm{\vc}^2\right)$ means
{\small\begin{align*}
    & \norm{\frac{\ty_{k,0}-\ty_{k,1}}{\lambda}}^2 \\
= & \norm{\frac{\ty_{k,0}-\ty_{k,1}}{\lambda} - \fG_{\lambda,\vy}(\vx_{k+1},\vy_k) + \fG_{\lambda,\vy}(\vx_{k+1},\vy_k) -  \fG_{\lambda,\vy}(\vx_k,\vy_k) + \fG_{\lambda,\vy}(\vx_k,\vy_k)}^2 \\ 
\leq & 3\norm{\frac{\ty_{k,0}-\ty_{k,1}}{\lambda} - \fG_{\lambda,\vy}(\vx_{k+1},\vy_k)}^2 + 3\norm{\fG_{\lambda,\vy}(\vx_{k+1},\vy_k) -  \fG_{\lambda,\vy}(\vx_k,\vy_k)}^2 + 3\norm{\fG_{\lambda,\vy}(\vx_k,\vy_k)}^2. 
\end{align*}}
We use Lemma \ref{lem:nonexpansive} and Lemma \ref{lem:diff-grad-mapping} to bound the first and the second term respectively, that is
\begin{align*}
    & \norm{\frac{\ty_{k,0}-\ty_{k,1}}{\lambda} - \fG_{\lambda,\vy}(\vx_{k+1},\vy_k)}^2 \\
=& \norm{\frac{1}{\lambda}(\vy_k-\Pi_{\fY}(\vy_k-\lambda\tu_{k,0})) - \frac{1}{\lambda}(\vy_k-\Pi_{\fY}(\vy_k-\lambda\nabla_\vy f(\vx_{k+1},\vy_k)))}^2 \\
=& \frac{1}{\lambda^2}\norm{\Pi_{\fY}(\vy_k-\lambda\tu_{k,0}) - \Pi_{\fY}(\vy_k-\lambda\nabla_\vy f(\vx_{k+1},\vy_k))}^2 \\
\leq & \norm{\tu_{k,0}-\nabla_\vy f(\vx_{k+1},\vy_k)}^2 
=  \tDelta_{k,0},
\end{align*}
and
\begin{align*}
    \norm{\fG_{\lambda,\vy}(\vx_{k+1},\vy_k) -  \fG_{\lambda,\vy}(\vx_k,\vy_k)}^2 \leq \ell^2\eps_\vx^2.
\end{align*}
The third term is $\norm{\fG_{\lambda,\vy}(\vx_k,\vy_k)}^2=\delta_k$ because of the definition. Hence, we have
\begin{align*}
      \norm{\frac{\ty_{k,0}-\ty_{k,1}}{\lambda}}^2 
\leq & 3(\tDelta_{k,0}+\ell^2\eps_\vx^2+\delta_k) \\
\leq & 3(\Delta_{k}+\frac{\ell^2\eps_\vx^2}{S_2}+\ell^2\eps_\vx^2+\delta_k) \\
\leq & 3(\Delta_{k}+2\ell^2\eps_\vx^2+\delta_k).
\end{align*}
\end{proof}

Then we can establish the recursive relationship of $\Delta_k$ and $\delta_k$. 
\begin{lem}\label{lem:recursion}
For Algorithm \ref{alg:SREDA} and \ref{alg:inner} with $\lambda=\frac{1}{8\ell}$, for any $k= k_0+1, k_0+2, \dots, k_0+q-1$, we have
\begin{align*}
    & \Delta_k
\leq  \Delta_{k_0} + \frac{3}{64S_2(1-\alpha)} \sum_{i=k_0}^{k-1}\left(\Delta_{i} + \delta_{i} + 2\ell^2\eps_\vx^2\right) + \frac{(k-k_0)\ell^2\eps_\vx^2}{S_2}, \\
    & \delta_{k+1}  
\leq \left(\frac{128\ell}{m\mu}+\frac{3}{8}\right)\delta_k + \left(\frac{64\ell}{m\mu}+\frac{67}{8}\right)\Delta_k +  \left(\frac{192\ell}{m\mu}+\frac{35}{4}\right)\ell^2\eps_\vx^2,
\end{align*}
where 
\begin{align*}
    \alpha=1 -\frac{2\lambda\mu \ell}{\mu+\ell}.
\end{align*}
\end{lem}
\begin{proof}
We define 
\begin{align*}
    \tDelta_{k,t}=\BE\left(\norm{\tv_{k,t}-\nabla_\vx f(\tx_{k,t},\ty_{k,t})}^2+\norm{\tu_{k,t}-\nabla_\vy f(\tx_{k,t},\ty_{k,t})}^2\right),
\end{align*}
then we have
{\small\begin{align}\label{ieq:tDelta}
\begin{split}
  & \tDelta_{k,0} \\  
= & \BE\left(\norm{\tv_{k,0}-\nabla_\vx f(\tx_{k,0},\ty_{k,0})}^2+\norm{\tu_{k,0}-\nabla_\vy f(\tx_{k,0},\ty_{k,0})}^2 \right) \\
\leq & \BE\left(\norm{\vv_k-\nabla_\vx f(\vx_{k},\vy_{k})}^2+\norm{\vu_{k}-\nabla_\vy f(\vx_{k},\vy_{k})}^2 \right)
    + \frac{\ell^2}{S_2}\BE\left(\norm{\tx_{k,0}-\vx_k}^2 + \norm{\ty_{k,0}-\vy_k}^2 \right) \\
= & \Delta_k +  \frac{\ell^2}{S_2}\BE\left(\norm{\vx_{k+1}-\vx_k}^2 + \norm{\vy_{k}-\vy_k}^2 \right) \\
\leq & \Delta_k + \frac{\ell^2\eps_\vx^2}{S_2},
\end{split}
\end{align}}
where the first inequality comes from Lemma \ref{lem:martingale} by letting $\fB(\cdot)=\nabla f(\cdot)$ and $\fV_k=(\vv_k,\vu_k)$.

Now for any $k\geq 1$, we have
\begin{align*}
      \Delta_k 
=   & \BE\left(\norm{\vv_{k}-\nabla_\vx f(\vx_{k},\vy_{k})}^2 + \norm{\vu_{k}-\nabla_\vy f(\vx_{k},\vy_{k})}^2\right) \\
=   & \tDelta_{k-1,s_{k-1}+1} \\
    \leq   &  \BE\left(\norm{\tv_{k-1,0}-\nabla_\vy f(\tx_{k-1,0},\ty_{k-1,0})}^2
    + \norm{\tu_{k-1,0}-\nabla_\vy f(\tx_{k-1,0},\ty_{k-1,0})}^2\right) \\
    &   + \frac{\ell^2}{S_2}\sum_{t=0}^{s_{k-1}}\left(\norm{\tx_{k-1,t+1}-\tx_{k-1,t}}^2 + \norm{\ty_{k-1,t+1}-\ty_{k-1,t}}^2\right) \\
= & \tDelta_{k-1,0} + \frac{\ell^2}{S_2} \sum_{t=0}^{s_{k-1}}\norm{\ty_{k-1,t+1}-\ty_{k-1,t}}^2  \\
\leq & \tDelta_{k-1,0} + \frac{\ell^2}{S_2}\sum_{t=0}^{s_k-1}\alpha^t\norm{\ty_{k-1,1}-\ty_{k-1,0}}^2 \\
\leq & \tDelta_{k-1,0} + \frac{\ell^2\lambda^2}{S_2(1-\alpha)} \norm{\frac{\ty_{k-1,1}-\ty_{k-1,0}}{\lambda}}^2,
\end{align*}
where the first inequality is due to Lemma \ref{lem:martingale}; the second inequality comes from Lemma \ref{lem:progress-decay} and the third one is due to basic property of geometric sequence.

Combining above results and Lemma \ref{lem:diff-y0-y1}, we have
\begin{align*}
      \Delta_k
\leq & \tDelta_{k-1,0} + \frac{\ell^2\lambda^2}{S_2(1-\alpha)} \norm{\frac{\ty_{k-1,1}-\ty_{k-1,0}}{\lambda}}^2 \\
\leq & \Delta_{k-1} + \frac{\ell^2\eps_\vx^2}{S_2}
        + \frac{3\ell^2\lambda^2}{S_2(1-\alpha)} \left(\Delta_{k-1} + \delta_{k-1} + 2\ell^2\eps_\vx^2\right).
\end{align*}

Summing over the above inequality from $k_0$ to $k$, we can prove the first part of this theorem as follows 
\begin{align*}
  \Delta_k
\leq & \Delta_{k-1} + \frac{3\ell^2\lambda^2}{S_2(1-\alpha)} \left(\Delta_{k-1} + \delta_{k-1} + 2\ell^2\eps_\vx^2\right)
      + \frac{\ell^2\eps_\vx^2}{S_2} \\
      \leq & \Delta_{k_0} + \frac{3}{64S_2(1-\alpha)} \sum_{i=0}^{k-1}\left(\Delta_{i} + \delta_{i} + 2\ell^2\eps_\vx^2\right)
      + \frac{(k-k_0)\ell^2\eps_\vx^2}{S_2}.
\end{align*}
Recall that $\vy_{k+1}=\ty_{k,s_k}$ and the definition of $g_k$, we achieve the second part of this lemma:
\begin{align*}
  &   \delta_{k+1} \\
= & \BE\norm{\fG_{\lambda,\vy} (\vx_{k+1},\vy_{k+1})}^2 \\
\leq &  \frac{64\ell}{m\mu}\BE\norm{\fG_{\lambda,\vy} (\vx_{k+1},\vy_{k})}^2
    + \left(8+\frac{64\ell}{m\mu}\right)\BE\norm{\nabla f(\tx_{k,0},\ty_{k,0})-\tu_{k,0}}^2 
    + 8\ell^2\BE\norm{\ty_{k,1}-\ty_{k,0}}^2 \\
\leq & \frac{128\ell}{m\mu}
    \BE\left(\norm{\fG_{\lambda,\vy} (\vx_{k+1},\vy_{k})-\fG_{\lambda,\vy} (\vx_{k},\vy_{k})}^2+\norm{\fG_{\lambda,\vy} (\vx_{k},\vy_{k})}^2\right)
    \\
& +\left(8+\frac{64\ell}{m\mu}\right)\tDelta_{k,0} + \frac{3}{8}(\Delta_k+\delta_k+2\ell^2\eps_\vx^2) \\
\leq & \frac{128\ell}{m\mu} \left(\ell^2\eps_\vx^2 + \delta_k \right) 
    +  \left(8+\frac{64\ell}{m\mu}\right)\tDelta_{k,0} + \frac{3}{8}(\Delta_k+\delta_k+2\ell^2\eps_\vx^2)  \\
\leq & \frac{128\ell}{m\mu} \left(\ell^2\eps_\vx^2 + \delta_k \right) +  \left(8+\frac{64\ell}{m\mu}\right)\left(\Delta_k + \frac{\ell^2\eps_\vx^2}{S_2}\right) + \frac{3}{8}(\Delta_k+\delta_k+2\ell^2\eps_\vx^2) \\
= & \left(\frac{128\ell}{m\mu}+\frac{3}{8}\right)\delta_k + \left(\frac{64\ell}{m\mu}+\frac{67}{8}\right)\Delta_k +  \left(\frac{192\ell}{m\mu}+\frac{35}{4}\right)\ell^2\eps_\vx^2,
\end{align*}
where the first inequality is according to Lemma \ref{lem:SR},
the second inequality is based on Young's inequality and Lemma \ref{lem:diff-y0-y1}, the third inequality is due to
Lemma \ref{lem:diff-grad-mapping} and the other steps are based on definitions.
\end{proof}

Now we can provide the upper bound of $\Delta_k$ and $\delta_k$.
\begin{cor}
    For Algorithm \ref{alg:SREDA} and Algorithm \ref{alg:inner} with
    \begin{align*}
        & \eta_k=\min\left(\frac{\eps}{5\kappa\ell\norm{\vv_k}},\frac{1}{10\kappa\ell}\right), \lambda=\frac{1}{8\ell},  
        m=\lceil1024\kappa\rceil, 
        q=\left\lceil\eps^{-1}\right\rceil, \\ 
        & S_1=\left\lceil\frac{2250}{19}\sigma^2\kappa^{-2}\eps^2\right\rceil, S_2=\left\lceil\frac{3687}{76}\kappa q\right\rceil, \text{~~and~~} \delta_0\leq\kappa^{-2}\eps^2,
    \end{align*}
    Then we have $\Delta_k\leq\frac{19}{1125}\kappa^{-2}\eps^2$ and $\delta_k\leq\kappa^{-2}\eps^2$ for all $k\geq 0$. 
\end{cor}
\begin{proof}
Firstly, we let $k_0$ be the number of round satisfies ${\rm mod}(k_0, q)=0$ in Algorithm \ref{alg:SREDA} and $\alpha$ be the one defined in Lemma \ref{lem:SR}, such that
\begin{align}
    \alpha 
= 1 - \frac{2\lambda\mu\ell}{\mu+\ell}
\leq 1-\frac{1}{8\kappa}. \label{ieq:alpha}
\end{align}
The choice of $S_1$ indicates we have 
\begin{align}
    \Delta_{k_0}<\frac{19}{2250}\kappa^{-2}\eps^2 \label{ieq:Delta0}
\end{align}
for all $k_0$. Then we prove the statement by induction. 
\paragraph{Induction base:}
The choice of $S_1$ and Assumption \ref{asm:grad} means
\begin{align*}
    \Delta_{k_0} \leq \frac{19}{2250}\kappa^{-2}\eps^2 < \frac{19}{1125}\kappa^{-2}\eps^2.
\end{align*}
Combing the assumptions of $\delta_0$, we obtain the induction base.

\paragraph{Induction step:}
For any $k\geq 1$, we suppose $\Delta_{k'}\leq\frac{19}{1125}\kappa^{-2}\eps^2$ and $\delta_{k'}\leq\kappa^{-2}\eps^2$ holds for all $k'< k$. 
Let $k'_0$ be the largest integer such that ${\rm mod}(k'_0, q)=0$ and $k'_0\leq k$.
Using Lemma \ref{lem:SR}, and inequalities (\ref{ieq:alpha}) and (\ref{ieq:Delta0}), we have 
\begin{align}\label{ieq:Delta-k}
\begin{split}
     \Delta_k
\leq & \Delta_{k'_0} + \frac{3}{64S_2(1-\alpha)} \sum_{i=k_0}^{k-1}\left(\Delta_{i} + \delta_{i} + 2\ell^2\eps_\vx^2\right) + \frac{(k-k_0)\ell^2\eps_\vx^2}{S_2} \\
\leq & \Delta_{k'_0} + \frac{3q}{64S_2(1-\alpha)} \left(\frac{19}{1125}\kappa^{-2}\eps^2 + \kappa^{-2}\eps^2 + \frac{2}{25}\kappa^{-2}\eps^2\right) + \frac{q}{25S_2}\kappa^{-2}\eps^2 \\
%\leq & \Delta_{k_0} + \frac{3\kappa q}{8S_2} \frac{1 109}{1125}\kappa^{-2}\eps^2 + \frac{\kappa q}{25S_2}\kappa^{-2}\eps^2 \\
%= & \Delta_{k_0} + \frac{1229\kappa q}{3000S_2}\kappa^{-2}\eps^2
%\leq \frac{19}{2250}\kappa^{-2}\eps^2+\frac{19}{2250}\kappa^{-2}\eps^2 \\
\leq & \frac{19}{1125}\kappa^{-2}\eps^2
\end{split}
\end{align}
and
\begin{align}\label{ieq:delta-k}
\begin{split}
     \delta_{k}  
\leq & \left(\frac{128\ell}{m\mu}+\frac{3}{8}\right)\delta_{k-1} + \left(\frac{67}{8}+\frac{64\ell}{m\mu}\right)\Delta_{k-1} + \left(\frac{35}{4}+\frac{192\ell}{m\mu}\right)\ell^2\eps_\vx^2 \\
\leq & \left(\frac{1}{8}+\frac{3}{8}\right)\delta_{k-1} + \left(\frac{67}{8}+\frac{1}{16}\right)\Delta_{k-1} +  \left(\frac{35}{4}+\frac{3}{16}\right)\cdot\frac{\kappa^{-2}\eps^2}{25} %\\
%= & \frac{1}{2}\kappa^{-2}\eps^2 + \frac{135}{16}\Delta_k +  \frac{143}{400}\kappa^{-2}\eps^2 \\
%\leq & \frac{1}{2}\kappa^{-2}\eps^2 + \frac{135}{16}\cdot\frac{19}{1125}\kappa^{-2}\eps^2 +  \frac{143}{400}\kappa^{-2}\eps^2 
\leq \kappa^{-2}\eps^2.
\end{split}
\end{align}
\end{proof}

\section{Initialization via Projected Inexact SARAH}\label{app:PiSARAH}
The initialization of SREDA (line \ref{line:PiSARAH} of Algorithm \ref{alg:SREDA}) can be regarded as solving a stochastic constrained convex minimization (concave maximization) problem. Hence, we consider the following formulation
\begin{align}
    \min_{\vw\in\fC} g(\vw) \triangleq \BE[G(\vw;\vxi)], \label{prob:convex}
\end{align}
where $\fC\subseteq\BR^d$ is a compact and convex set and $\vxi$ is a random variable. 

We suppose the optimal solution is $\vw^*=\argmin_{\vw\in\fC}g(\vw)$, the condition number is $\kappa=\mu/\ell$ and the following assumptions hold.
\begin{asm}\label{asm:G}
    The component function $G$ has an average $\ell$-Lipschitz gradient,  i.e.,
    there exists a constant $\ell>0$ such that for any $\vw$, $\vw'$ and random vector $\vxi$, we have
    \begin{align*}
        \BE\norm{\nabla G(\vw; \vxi) - \nabla F(\vw'; \vxi)}^2
        \leq &\ell^2\norm{\vw-\vw'}^2.
    \end{align*}
\end{asm}

\begin{asm}\label{asm:G2}
    The component function $G$ is convex. That is, for any $\vw$, $\vw'$ and random vector $\vxi$, we have
    \begin{align*}
        G(\vw; \vxi) \geq G(\vw'; \vxi) + \inner{\nabla G(\vw';\vxi)}{\vw-\vw'}.
    \end{align*}
\end{asm}

\begin{asm}\label{asm:g}
    The function $g(\vw)$ is $\mu$-strongly-convex. That is, there exists a constant $\mu>0$ such that for any  $\vw$ and $\vw'$, we have
    \begin{align*}
        g(\vw) \geq g(\vw') + \inner{\nabla g(\vw')}{\vw-\vw'} + \frac{\mu}{2}\norm{\vw-\vw'}^2.
    \end{align*}
\end{asm}

\begin{asm}\label{asm:g-grad}
    The gradient of each component function $G(\vw;\vxi)$ has bounded variance. That is, there exists a constant $\sigma>0$ such that for and $\vw$ and random vector $\vxi$, we have
    \begin{align*}
        \BE\norm{\nabla G(\vw;\vxi) - \nabla g(\vw)}^2 \leq \sigma^2 < \infty.
    \end{align*}
\end{asm}

We propose projected inexact SARAH (PiSARAH)  to solve problem (\ref{prob:convex}), whose detailed procedure is presented in Algorithm \ref{alg:PiSARAH}.
%There are several difference between PiSARAH and iSARAH~\cite{nguyen2018inexact}:
%\begin{enumerate}
%    \item PiSARAH considers the constraint convex optimization by conducting a projection step, while iSARAH only addresses the unconstrained problem. 
%    \item The convergence analysis of PiSARAH only depends on the average smooth condition (Assumption \ref{asm:G}), while iSARAH requires each component function is smooth.
%    \item The convergence analysis of PiSARAH requires the stochastic gradient has bounded variance (Assumption \ref{asm:g-grad}), which is unnecessary for iSARAH.
%\end{enumerate}

\begin{algorithm}[ht]
    \caption{${\rm PiSARAH}~(g(\cdot), K_0)$}\label{alg:PiSARAH}
	\begin{algorithmic}[1]
	\STATE \textbf{Input} $\vw_0\in\fC$, learning rate $\gamma>0$, inner loop size $m$, batch sizes $b_1$ \\[0.1cm]
    \STATE \textbf{for}  $k=0,\dots,K_0-1$ \textbf{do} \\[0.1cm]
    \STATE\quad draw $b_1$ samples $\{\vxi_1,\cdots,\vxi_b\}$ \\[0.1cm]
    \STATE\quad $\tw_{k,0} = \vw_{k}$ \\[0.1cm]
    \STATE\quad $\tv_{k,0} = \frac{1}{b_1}\sum_{i=1}^{b_1} \nabla G(\tw_{k,0}; \vxi_{i})$ \\[0.1cm]
    \STATE\quad $\tw_{k,1} = \tw_{k,0} - \gamma\tv_{k,0}$ \\[0.1cm]
    \STATE\quad \textbf{for}  $t=1,\dots,m-1$ \textbf{do} \\[0.1cm]
    \STATE\quad\quad draw sample $\vxi_t$ \\[0.1cm] 
    \STATE\quad\quad $\tv_{k,t} = \tv_{k,t-1} + \nabla G(\tw_{k,t};\vxi_t) - \nabla G(\tw_{k,t-1};\vxi_t) $ \\[0.1cm]
    \STATE\quad\quad $\tw_{k,t+1} = \Pi_\fC(\tw_{k,t} - \gamma\tv_{k,t}$)  \\[0.1cm]
    \STATE\quad \textbf{end for} \\[0.1cm]
    \STATE\quad $\vw_{k+1}=\tw_{k,s_k}$, where $s_k$ is uniformly sampled from $\{1,\dots,m\}$ \\[0.1cm]
    \STATE \textbf{end for} \\[0.1cm]
	\STATE \textbf{Output}: $\vw_{K_0}$
	\end{algorithmic}
\end{algorithm}

We are interested in the convergence behavior of the gradient mapping, that is
\begin{align*}
    \fG_{\gamma}(\vw)\triangleq \frac{\vw-\Pi_{\fC }(\vw-\lambda\nabla g(\vw))}{\lambda}.
\end{align*}

The remain of this section provide the convergence analysis of PiSARAH. 

Note that each epoch of PiSARAH can be regarded as using ConcaveMaximizer (Algorithm \ref{alg:inner}) on $-g(\cdot)$. Hence, we can follow the analysis of Lemma \ref{lem:SR} to achieve the result as follows. 
\begin{cor}\label{cor:SR}
For Algorithm \ref{alg:PiSARAH} with $\gamma=\frac{1}{8\ell}$, we have 
{\begin{align*}
   \BE\|\fG_{\gamma}(\vw_{k+1})\|_2^2  
\leq \left(\frac{64\ell}{m\mu}+\frac{1}{4}\right)\BE\|\fG_{\gamma}(\vw_{k})\|_2^2 + \left(\frac{64\ell}{m\mu}+\frac{33}{4}\right)\BE\norm{\nabla g(\tw_{k,0})-\tv_{k,0}}^2.
\end{align*}}
\end{cor}
\begin{proof}
Using Lemma \ref{lem:SR} in the view of $g_k(\cdot)=g(\cdot)$, we have
\begin{align}\label{ieq:SR-convex}
\begin{split}
   & \BE\|\fG_{\gamma}(\vw_{k+1})\|_2^2 \\  
\leq & \frac{64\ell}{m\mu}\BE\|\fG_{\gamma}(\vw_{k})\|_2^2 + \left(\frac{64\ell}{m\mu}+8\right)\BE\norm{\nabla g(\tw_{k,0})-\tv_{k,0}}^2 + 8\ell^2\BE\norm{\tw_{k,1}-\tw_{k,0}}^2.
\end{split}
\end{align}
The last term of (\ref{ieq:SR-convex}) can be bounded by
\begin{align*}
    & \norm{\frac{\tw_{k,0}-\tw_{k,1}}{\gamma}}^2 \\
\leq  & 2\norm{\frac{\tw_{k,0}-\tw_{k,1}}{\gamma} - \fG_\gamma(\vw_{k,0})}^2 + 2\norm{\fG_\gamma(\vw_{k,0})}^2 \\
=  & 2\norm{\frac{\tw_{k,0}-\Pi_\fC(\vw_{k,0}-\lambda\tv_{k,0})}{\gamma} - \frac{\vw_{k,0}-\Pi_\fC(\vw_{k,0}-\lambda\nabla g(\vw_{k,0}))}{\gamma}}^2 + 2\norm{\fG_\gamma(\vw_{k,0})}^2 \\
=  & 2\norm{\tv_{k,0}-\nabla g(\vw_{k,0})}^2 + 2\norm{\fG_\gamma(\vw_{k,0})}^2,
\end{align*}
which implies
\begin{align}\label{ieq:SR-convex2}
    8\ell^2\BE\norm{\tw_{k,1}-\tw_{k,0}}^2
\leq \frac{1}{4}\BE\left[\norm{\tv_{k,0}-\nabla g(\vw_{k,0})}^2 + \norm{\fG_\gamma(\vw_{k,0})}^2\right].
\end{align}
We finish the proof by combining (\ref{ieq:SR-convex}) and (\ref{ieq:SR-convex2}).
\end{proof}

Then we provide the main result in this section to show the convergence of the gradient mapping.
\begin{thm}\label{thm:PiSARAH}
For Algorithm \ref{alg:PiSARAH} with
\begin{align*}
K_0=\left\lceil\frac{\log\left(2\zeta^{-1}\|\fG_{\gamma}(\vw_{0})\|_2^2\right)}{\log2}\right\rceil, m=\left\lceil{256\kappa}\right\rceil, 
\lambda = \frac{1}{8\ell} \text{~~and~~} 
b_1=\left\lceil34\sigma^2\zeta^{-1}\right\rceil. 
\end{align*}
Then we have $\BE\|\fG_{\gamma}(\vw_{K_0})\|_2^2 \leq \zeta$.
\end{thm}
\begin{proof}
Using Corollary \ref{cor:SR} with $m=\left\lceil256\kappa\right\rceil$ and $b_1=\left\lceil34\sigma^2\zeta^{-1}\right\rceil$
we have
\begin{align*}
    \BE\|\fG_{\gamma}(\vw_{k+1})\|_2^2  
\leq \frac{1}{2}\BE\|\fG_{\gamma}(\vw_{k})\|_2^2 + \frac{\zeta}{4},
\end{align*}
which implies
\begin{align*}
     & \BE\|\fG_{\gamma}(\vw_{K_0})\|_2^2-\frac{\zeta}{2} \\
\leq & \frac{1}{2}\left(\BE\|\fG_{\gamma}(\vw_{K_0-1})\|_2^2 - \frac{\zeta}{2}\right) \\
\leq & \frac{1}{2^{K_0}}\left(\BE\|\fG_{\gamma}(\vw_{0})\|_2^2 - \frac{\zeta}{2}\right) \\
\leq & \frac{1}{2^{K_0}}\BE\|\fG_{\gamma}(\vw_{0})\|_2^2
\leq \frac{\zeta}{2}.
\end{align*}
Hence, we have $\BE\|\fG_{\gamma}(\vw_{K_0})\|_2^2 \leq \zeta$. 
\end{proof}

The result of Corollary \ref{cor:SR} indicate that we hope the PiSARAH as initialization to make the gradient mapping is no larger than $\fO(\kappa^{-2}\eps^2)$. The following statement shows we can implement it within $\fO(\kappa^2\eps^{-2}\log(\kappa/\eps))$ stochastic gradient evaluations. 

\begin{cor}
    Under assumptions of Theorem \ref{thm:PiSARAH}, we can obtain $\BE\|\fG_{\gamma}(\vw_{K_0})\|_2^2 \leq \kappa^{-2}\eps^{2}$ with $\fO(\kappa^2\eps^{-2}\log(\kappa/\eps))$ stochastic gradient evaluations.
\end{cor}
\begin{proof}
Using Theorem \ref{thm:PiSARAH} with $\zeta=\kappa^{-2}\eps^{2}$, we have $\BE\|\fG_{\gamma}(\vw_{K_0})\|_2^2 \leq \kappa^{-2}\eps^{2}$.
The total number of stochastic gradient evaluation is
\begin{align*}
  & K_0 \cdot (b_1 + m) \\ 
= & \left\lceil\frac{\log\left(2\kappa^2\eps^{-2}\|\fG_{\gamma}(\vw_{0})\|_2^2\right)}{\log2}\right\rceil \cdot \left(\left\lceil34\sigma^2\kappa^2\eps^{-2}\right\rceil + \left\lceil{256\kappa}\right\rceil \right) \\
= & \fO(\kappa^2\eps^{-2}\log(\kappa/\eps))
\end{align*}
\end{proof}

\section{The Proof of Theorem \ref{thm:main}}\label{app:main}

Our proof mainly depends on $f(\vx_k, \vy_k)$ and its gradient mapping with respect to $\vy$, which is
different from \citeauthor{lin2019gradient}'s~\cite{lin2019gradient} analysis that directly considered the value of
$\Phi(\vx_k)$ and the distance $\norm{\vy_k-\vy^*(\vx_k)}$. We split
the change of objective functions after one iteration on $(\vx_k, \vy_k)$ into $A_k$ and $B_k$ as follows
\begin{align*}
    f(\vx_{k+1}, \vy_{k+1}) - f(\vx_{k}, \vy_{k}) 
 =  \underbrace{f(\vx_{k+1}, \vy_{k}) - f(\vx_{k}, \vy_{k})}_{A_k} + \underbrace{f(\vx_{k+1}, \vy_{k+1})- f(\vx_{k+1}, \vy_{k})}_{B_k},
\end{align*}
where $A_k$ provides the decrease of function value $f$ and $B_k$ can characterize the difference between $f(\vx_{k+1},\vy_{k+1})$ and $\Phi(\vx_{k+1})$. 
We want to show $\BE[A_k] \leq - \fO\left(\kappa^{-1}\eps\right)$
and
$\BE[B_k] \leq \fO\left((\kappa\ell)^{-1}\eps^2\right)$.
Connecting the upper bound of $A_k$ and $B_k$, we can bound the average of $\BE\norm{\vv_k}^2$ and use it to prove the upper bound of $\BE[\nabla\Phi(\hat\vx)]$ we desired.

We provide two lemmas for preparing the proof of our main results, Theorem \ref{thm:main}. The first lemma is to upper bound $B_k$.
\begin{lem}\label{lem:bound-B}
    Under assumptions of Theorem 1, we have
    $\BE[B_k] \leq  \dfrac{134\eps^2}{\kappa\ell}$ for any $k\geq 1$.
\end{lem}
\begin{proof}
Note that our algorithm means $\vy_k=\ty_{k-1,s_{k-1}}$, where $s_{k-1}$ is sampled from $\{1,\dots,m\}$.
Using Lemma 8 by letting $\vy^+=\vy_k=\ty_{k-1,s_{k-1}}$, $\vy=\ty_{k-1,s_{k-1}-1}$ and $\vu=\tu_{k-1,s_{k-1}-1}=-\hu_{k-1,s_{k-1}-1}$, then we have
\begin{align}
  & \BE[f(\vx_{k+1},\vy_{k+1})-f(\vx_{k+1},\vy_{k})]\nonumber\\
= & \BE[f(\vx_{k+1},\vy_{k+1})-f(\vx_{k+1},\ty_{k-1,s_{k-1}})] \nonumber\\
\leq & \frac{1}{\mu}\BE\left[\norm{\fG_{\lambda,\vy}(\vx_{k+1},\ty_{k-1,s_{k-1}-1})}^2+\norm{\nabla_\vy f(\vx_{k+1},\ty_{k-1,s_{k-1}-1})-\tu_{k-1,s_{k-1}-1}}^2\right]\label{ieq:B}.
\end{align}
We first bound the first term of (\ref{ieq:B}):
\begin{align*}
  & \BE\norm{\fG_{\lambda,\vy}(\vx_{k+1},\ty_{k-1,s_{k-1}-1})}^2 \\
\leq & 2\BE\norm{\fG_{\lambda,\vy}(\vx_{k+1},\vy_{k})}^2 + 2\BE\norm{\fG_{\lambda,\vy}(\vx_{k+1},\ty_{k-1,s_{k-1}-1})-\fG_{\lambda,\vy}(\vx_{k+1},\ty_{k-1,s_{k-1}})}^2 \\
= & 2\BE\left[\norm{\fG_\lambda(\vx_{k+1},\vy_{k})-\fG_\lambda(\vx_{k},\vy_{k})}+\norm{\fG_\lambda(\vx_{k},\vy_{k})}^2\right]  \\
  & + 2\BE\norm{\fG_\lambda(\vx_{k+1},\ty_{k-1,s_{k-1}-1})-\fG_\lambda(\vx_{k+1},\ty_{k-1,s_{k-1}})}^2\\
\leq & 2(\ell^2\eps_\vx^2 + \delta_k) +  2\BE\norm{\fG_\lambda(\vx_{k+1},\ty_{k-1,s_{k-1}-1})-\fG_\lambda(\vx_{k+1},\ty_{k-1,s_{k-1}})}^2,
\end{align*}
where the inequalities are based on Young's inequality and Lemma \ref{lem:diff-grad-mapping}. 

Using similar ideas, we can also prove
\begin{align*}
&  \BE\norm{\fG_{\lambda,\vy}(\vx_{k+1},\ty_{k-1,s_{k-1}-1})-\fG_{\lambda,\vy}(\vx_{k+1},\ty_{k-1,s_{k-1}})}^2 \\
\leq &  3\BE\norm{\fG_{\lambda,\vy}(\vx_{k+1},\ty_{k-1,s_{k-1}-1})-\fG_{\lambda,\vy}(\vx_{k},\ty_{k-1,s_{k-1}-1})}^2 \\
& +  3\BE\norm{\fG_{\lambda,\vy}(\vx_{k+1},\ty_{k-1,s_{k-1}})-\fG_{\lambda,\vy}(\vx_{k},\ty_{k-1,s_{k-1}})}^2  \\
& + 3\BE\norm{\fG_{\lambda,\vy}(\vx_{k},\ty_{k-1,s_{k-1}-1})-\fG_{\lambda,\vy}(\vx_{k},\ty_{k-1,s_{k-1}})}^2 \\
\leq & 6(\ell^2\eps_\vx^2+\delta_k) + 3\BE\norm{\fG_{\lambda,\vy}(\vx_{k},\ty_{k-1,s_{k-1}-1})-\fG_{\lambda,\vy}(\vx_{k},\ty_{k-1,s_{k-1}})}^2,
\end{align*}
and
{\begin{align*}
& \BE\norm{\fG_{\lambda,\vy}(\vx_{k},\ty_{k-1,s_{k-1}-1})-\fG_{\lambda,\vy}(\vx_{k},\ty_{k-1,s_{k-1}})}^2 \\
= & \BE\Bigg\|\frac{\ty_{k-1,s_{k-1}-1}-\Pi_{\fY}(\ty_{k-1,s_{k-1}-1}-\lambda\nabla g_k(\ty_{k-1,s_{k-1}-1}))}{\lambda} \\
& \quad\quad\quad - \frac{\ty_{k-1,s_{k-1}}-\Pi_{\fY}(\ty_{k-1,s_{k-1}}-\lambda\nabla g_k(\ty_{k-1,s_{k-1}}))}{\lambda}\Bigg\|_2^2 \\
\leq &  2\BE\Bigg\|\frac{\Pi_{\fY}(\ty_{k-1,s_{k-1}-1}-\lambda\nabla g_k(\ty_{k-1,s_{k-1}-1}))}{\lambda} 
-\frac{\Pi_{\fY}(\ty_{k-1,s_{k-1}}-\lambda\nabla g_k(\ty_{k-1,s_{k-1}}))}{\lambda}\Bigg\|_2^2 \\
& + 2\BE\norm{\frac{\ty_{k-1,s_{k-1}-1}-\ty_{k-1,s_{k-1}}}{\lambda}}^2 \\
\leq &  2\BE\Bigg\|\frac{(\ty_{k-1,s_{k-1}-1}-\lambda\nabla g_k(\ty_{k-1,s_{k-1}-1}))}{\lambda} 
-\frac{(\ty_{k-1,s_{k-1}}-\lambda\nabla g_k(\ty_{k-1,s_{k-1}}))}{\lambda}\Bigg\|_2^2 \\
& + 2\BE\norm{\frac{\ty_{k-1,s_{k-1}-1}-\ty_{k-1,s_{k-1}}}{\lambda}}^2 \\
\leq & 6\BE\norm{\frac{\ty_{k-1,s_{k-1}-1}-\ty_{k-1,s_{k-1}}}{\lambda}}^2 + 
2\BE\norm{\nabla g_k(\ty_{k-1,s_{k-1}-1}))-\nabla g_k(\ty_{k-1,s_{k-1}}))}^2\\
\leq & (6+2\ell^2\lambda^2)\BE\norm{\frac{\ty_{k-1,s_{k-1}-1}-\ty_{k-1,s_{k-1}}}{\lambda}}^2 \\
\leq & (6+2\ell^2\lambda^2)\BE\norm{\frac{\ty_{k-1,1}-\ty_{k-1,0}}{\lambda}}^2. 
\end{align*}}
Combining all above results, we have
\begin{align}\label{ieq:B1}
\begin{split}
  & \BE\norm{\fG_{\lambda,\vy}(\vx_{k+1},\ty_{k-1,s_{k-1}-1})}^2 \\
\leq & 14(\ell^2\eps_\vx^2+\delta_k) + 6 (6+2\ell^2\lambda^2)\BE\norm{\frac{\ty_{k-1,1}-\ty_{k-1,0}}{\lambda}}^2 \\
\leq & 14(\ell^2\eps_\vx^2+\delta_k) + 18 (6+2\ell^2\lambda^2)(\Delta_{k}+2\ell^2\eps_\vx^2+\delta_k) \\
\leq & 14\left(\frac{1}{25}\kappa^{-2}\eps^2+\kappa^{-2}\eps^2\right) + \frac{1737}{16}\left(\frac{19}{1125}\kappa^{-2}\eps^2+\frac{2}{25}\kappa^{-2}\eps^2+\kappa^{-2}\eps^2\right) \\
\leq &  \left(\frac{364}{25} + \frac{1737}{16}\cdot\frac{1234}{1125}\right)\kappa^{-2}\eps^2 \\
= &  \frac{133641}{1000}\kappa^{-2}\eps^2
\end{split}
\end{align}
where the second inequality is based on Lemma \ref{lem:diff-y0-y1} and the third inequality is due to Corollary \ref{cor:SR}. 

We bound the second term of (\ref{ieq:B}) as follows:
\begin{align}\label{ieq:B2}
\begin{split}
& \BE\norm{\nabla_\vy f(\vx_k,\ty_{k-1,s_{k-1}-1})-\tu_{k-1,s_{k-1}-1}}^2 \\ \leq & \BE\norm{\nabla_\vy f(\vx_k,\ty_{k-1,s_{k-1}})-\tu_{k-1,s_{k-1}}}^2 \\
= & \BE\norm{\nabla_\vy f(\vx_k,\vy_{k})-\vu_{k}}^2 \\
\leq & \Delta_k \leq \frac{19}{1125}\kappa^{-2}\eps^2,
\end{split}
\end{align}
where the first inequality is based on Lemma \ref{lem:martingale} and the last one is due to Corollary \ref{cor:SR}.

By connecting inequalities (\ref{ieq:B}), (\ref{ieq:B1}) and (\ref{ieq:B2}), we have
\begin{align*}
    \BE[B_k] \leq \frac{1}{\mu}\cdot\frac{1202921}{9000}\kappa^{-2}\eps^2 \leq \frac{134\eps^2}{\kappa\ell}.
\end{align*}
\end{proof}

Then we show the estimate error of approximating $\nabla\Phi(\vx_k)$ by $\vv_k$.
\begin{lem}\label{diff:phi-v}
Under assumptions of Theorem \ref{thm:main}, we have
\begin{align*}
    \BE\norm{\nabla \Phi(\vx_k)}
    \leq \BE\norm{\vv_k} + \frac{15}{7}\eps.
    \end{align*}
\end{lem}
\begin{proof}
Consider that we have defined  $\vy^*(\vx)=\argmax_{\vy\in\fY} f(\vx,\vy)$, then we have
\begin{align*}
  & \BE\norm{\nabla\Phi(\vx_k)-\nabla_\vx f(\vx_k,\vy_k)}^2 \\
= & \BE\norm{\nabla_\vx f(\vx_k,\vy^*(\vx_k))-\nabla_\vx f(\vx_k,\vy_k)}^2 \\
\leq& \ell^2\BE\norm{\vy^*(\vx_k)-\vy_k}^2 \\
\leq & \frac{4\ell^2}{\mu^2}\BE\norm{\fG_{\lambda,\vy}(\vx_k,\vy_k)}^2 \\
=  & \frac{4\ell^2}{\mu^2}\delta_k 
\leq 4\eps^2,
\end{align*}
where the first equality is based on Lemma \ref{lem:Phi-smooth}, the second inequality comes from Corollary \ref{cor:prox2} and the last inequality is due to Lemma \ref{lem:SR}.

By using Jensen's inequality, we have
\begin{align*}
    \Big(\BE\norm{\nabla \Phi(\vx_k) - \nabla_\vx f(\vx_k,\vy_k)}\Big)^2
    \leq \BE\norm{\nabla \Phi(\vx_k) - \nabla_\vx f(\vx_k,\vy_k)}^2
    \leq 4\eps^2,
\end{align*}
which means
\begin{align}
\begin{split}\label{ieq:v-gphi1}
    \BE\norm{\nabla \Phi(\vx_k)}
= & \BE\norm{\nabla_\vx f(\vx_k,\vy_k)-(\nabla_\vx f(\vx_k,\vy_k)-\nabla \Phi(\vx_k))} \\
\leq & \BE\norm{\nabla_\vx f(\vx_k,\vy_k)} + \BE\norm{\nabla_\vx f(\vx_k,\vy_k)-\nabla \Phi(\vx_k)} \\
\leq & \BE\norm{\nabla_\vx f(\vx_k,\vy_k)} + 2\eps.
\end{split}
\end{align}
Similarly, we can use Jensen's inequality and Lemma \ref{lem:SR} to prove
\begin{align*}
    \Big(\BE\norm{\vv_k - \nabla_\vx f(\vx_k,\vy_k)}\Big)^2
\leq \BE\norm{\vv_k - \nabla_\vx f(\vx_k,\vy_k)}^2
\leq \Delta_k
\leq \frac{19}{1125}\kappa^{-2}\eps^2,
\end{align*}
and
\begin{align}
\begin{split}\label{ieq:v-gphi2}
    \BE\norm{\nabla_\vx f(\vx_k,\vy_k)}
= & \BE\norm{\vv_k-(\vv_k-\nabla_\vx f(\vx_k,\vy_k))} \\
\leq & \BE\norm{\vv_k} + \BE\norm{\vv_k-\nabla_\vx f(\vx_k,\vy_k)} \\
\leq & \BE\norm{\vv_k} +  \sqrt{\frac{19}{1125}}\kappa^{-1}\eps \\
\leq & \BE\norm{\vv_k} +  \frac{1}{7}\eps.
\end{split}
\end{align}
By combining the inequalities (\ref{ieq:v-gphi1}) and (\ref{ieq:v-gphi2}), we obtain
\begin{align*}
\BE\norm{\nabla \Phi(\vx_k)}
\leq & \BE\norm{\nabla_\vx f(\vx_k,\vy_k)} + 2\eps \\
\leq & \BE\norm{\vv_k} +  \frac{1}{7}\eps + 2\eps \\
\leq & \BE\norm{\vv_k} +  \frac{15}{7}\eps.
\end{align*}
\end{proof}

Now we can present the proof of Theorem \ref{thm:main}.
\begin{proof}
Based on the update of $\vx_k$ in Algorithm \ref{alg:SREDA}, we have
\begin{align}\label{ieq:f-smooth}
\begin{split}
          A_k 
    \leq  & -\eta_k\inner{\nabla_\vx f(\vx_k,\vy_k)}{\vv_k} + \frac{\ell\eta_k^2}{2}\norm{\vv_k}^2 \\
    \leq  & \frac{\eta_k}{2}\norm{\nabla_\vx f(\vx_k,\vy_k)-\vv_k}^2
        - \left(\frac{\eta_k}{2} - \frac{\ell\eta_k^2}{2}\right)\norm{\vv_k}^2,
\end{split}
\end{align}
where the first inequality is due to the average smoothness of $f$, and second comes from the Cauchy-Schwartz inequality.

The choice of step size $\eta_k$ implies that
{\begin{align}
\begin{split}
    \left(\frac{\eta_k}{2} - \frac{\ell\eta_k^2}{2}\right)\norm{\vv_k}^2 
\geq & \frac{9\eps^2}{100\kappa\ell} \min\left(\dfrac{\norm{\vv_k}}{\eps}, \dfrac{\norm{\vv_k}^2 }{2\eps^2}\right) \\
\geq & \frac{9\eps^2}{100\kappa\ell}  \left(\dfrac{\norm{\vv_k}}{\eps}-2\right)  \\
= & \frac{9}{100\kappa\ell} \left(\eps\norm{\vv_k}-2\eps^2\right),
\end{split}\label{ieq:eta}
\end{align}}
where the first inequality is based on $\kappa\geq 1$ and the
definition of $\eta_k$; the second one uses the fact that 
$\min(|x|, \frac{x^2}{2}) \geq |x|-2$ holds
for all $x$. 

By combining inequalities (\ref{ieq:f-smooth}), (\ref{ieq:eta}) and taking expectation, we obtain the upper bound of $\BE[A_k]$:
\begin{align}
\begin{split}
           \BE[A_k] 
    \leq  & \frac{1}{20\kappa\ell}\BE\norm{\nabla_\vx f(\vx_k,\vy_k)-\vv_k}^2 -  \frac{9}{100\kappa\ell} \left(\eps\BE\norm{\vv_k}-2\eps^2\right) \\
    \leq  & \frac{1}{20\kappa\ell}\Delta_k -  \frac{9}{100\kappa\ell} \left(\eps\BE\norm{\vv_k}-2\eps^2\right).
\end{split}\label{ieq:Ak0}
\end{align}

The definition of $\Phi^*$ and Assumption \ref{asm:bound} implies
\begin{align}
\begin{split}
      \Phi^* - f(\vx_K,\vy_K) 
\leq  f(\vx_K,\vy^*(\vx_K)) - f(\vx_K,\vy_K) 
\leq  \frac{134\eps^2}{\kappa\ell}.
\end{split}\label{ieq:Delta_f}
\end{align}
where the second inequality can be shown by following the proof of 
Lemma \ref{lem:bound-B}\footnote{Note that the proof of Lemma \ref{lem:bound-B} is based on inequality (\ref{ieq:B}) whose left-hand side can be replaced by $\BE[f(\vx_{k+1},\vy^*)-f(\vx_{k+1},\vy)]$ for any $\vy^*\in\fY$ because of Lemma \ref{lem:PL}. Hence, we can directly obtain the second inequality of (\ref{ieq:Delta_f}) by letting $k=K-1$ and $\vy^*=\vy^*(\vx_K)$.}.

By combining inequalities (\ref{ieq:Ak0}), (\ref{ieq:Delta_f}), Lemma \ref{lem:bound-B} and Corollary \ref{cor:SR}; and taking the average over $k=0,\dots,K-1$, we obtain
{\small\begin{align*}
      \frac{1}{K}\sum_{k=0}^{K-1} \BE[f(\vx_{k+1},\vy_{k+1})-f(\vx_k,\vy_k)]
\leq  \frac{1}{K}\sum_{k=0}^{K-1} \left(\frac{1}{20\kappa\ell}\Delta_k - \frac{9}{100\kappa\ell} \left(\eps\BE\norm{\vv_k}-2\eps^2\right) + \frac{134\eps^2}{\kappa\ell}\right).
\end{align*}}
Consequently, we have
{\begin{align*}
    & \frac{9\eps}{100\kappa\ell}\cdot\frac{1}{K}\sum_{k=0}^{K-1} \BE\norm{\vv_k} \\
\leq & \frac{1}{K}\sum_{k=0}^{K-1} \left(\frac{1}{20\kappa\ell}\Delta_k + \frac{9\eps^2}{50\kappa\ell}  + \frac{134\eps^2}{\kappa\ell}\right) +  \frac{1}{K}\Big(f(\vx_0,\vy_0)-\BE[f(\vx_{K},\vy_{K})]\Big) \\
\leq & \frac{135\eps^2}{\kappa\ell} +  \frac{1}{K}\Big(f(\vx_0,\vy_0)+\frac{134\eps^2}{\kappa\ell}-\Phi^*\Big),
\end{align*}}
where the second inequality uses Corollary \ref{cor:SR} to bound $\Delta_k$.

Rearranging above result, we achieve
{\begin{align}
    \frac{1}{K}\sum_{k=0}^{K-1} \BE\norm{\vv_k} 
\leq 1500\eps + \frac{100\kappa\ell}{9K\eps}\Big(\BE[f(\vx_0,\vy_0)]+\frac{134\eps^2}{\kappa\ell}-\Phi^*\Big) 
= 1500\eps + \frac{100\kappa\ell\Delta_f}{9K\eps}. \label{ieq:avg-vk}
\end{align}}
According to $K=\left\lceil100\kappa\ell\eps^{-2}\Delta_f/9\right\rceil$ and inequality (\ref{ieq:avg-vk}),
we have
\begin{align*}
  \BE\norm{\nabla\Phi(\hat\vx)}
= \frac{1}{K}\sum_{k=0}^{K-1}\BE\norm{\nabla\Phi(\vx_k)}
\leq \frac{1}{K}\sum_{k=0}^{K-1}\left(\BE\norm{\vv_k}+\frac{15}{7}\eps\right)
= 1504\eps.
\end{align*}
\end{proof}

\section{The proof of Theorem \ref{thm:finite}}\label{app:finite}

In the finite-sum case, we use the full gradient to replace the large batch sample size in stochastic case.
Similar to previous section, we extend SARAH~\cite{nguyen2017sarah} to constrained case as the initialization of $\vy_0$. We can prove Theorem \ref{thm:finite} with minor modifications on the analysis of Theorem \ref{thm:main}.

\begin{algorithm}[ht]
    \caption{${\rm PSARAH}~(g(\cdot),~K_0)$}\label{alg:PSARAH}
	\begin{algorithmic}[1]
	\STATE \textbf{Input} $\vw_0\in\fC$, learning rate $\gamma>0$, inner loop size $m$ \\[0.1cm]
    \STATE \textbf{for}  $k=1,\dots,K_0$ \textbf{do} \\[0.1cm]
    \STATE\quad $\tw_{k,0} = \vw_{s-1}$ \\[0.1cm]
    \STATE\quad $\tv_{k,0} =\nabla g(\tw_{k,0})$ \\[0.1cm]
    \STATE\quad $\tw_{k,1} = \tw_{k,0} - \gamma\tv_{k,0}$ \\[0.1cm]
    \STATE\quad \textbf{for}  $t=1,\dots,m-1$ \textbf{do} \\[0.1cm]
    \STATE\quad\quad draw sample $\vxi_t$ \\[0.1cm] 
    \STATE\quad\quad $\tv_{k,t} = \tv_{k,t-1} + \nabla G(\tw_{k,t};\vxi_t) - \nabla G(\tw_{k,t-1};\vxi_t) $ \\[0.1cm]
    \STATE\quad\quad $\tw_{k,t+1} = \Pi_\fC(\tw_{k,t} - \gamma\tv_{k,t}$) \label{line:PiSARAH-proj} \\[0.1cm]
    \STATE\quad \textbf{end for} \\[0.1cm]
    \STATE\quad $\vw_{k+1}=\tw_{k,s_k}$, where $s_k$ is uniformly sampled from $\{1,\dots,m\}$ \\[0.1cm]
    \STATE \textbf{end for} \\[0.1cm]
	\STATE \textbf{Output}: $\vw_{K_0}$
	\end{algorithmic}
\end{algorithm}

\subsection{Initialization by Projected SARAH}\label{app:SARAH}
We present the detailed procedure of projected SARAH (PSARAH) in Algorithm \ref{alg:PSARAH}, which is used to initialize $\vy_0$ in SREDA for problem (\ref{prob:main-finite}) (line \ref{line:SARAH} of Algorithm \ref{alg:SREDA-finite}). The algorithm considers the following convex optimization problem
\begin{align}
    \min_{\vw\in\fC} g(\vw) \triangleq \frac{1}{n}\sum_{i=1}^n G_i(\vw; \vxi_i), \label{prob:convex-finite}
\end{align}
where $H$ is average $\ell$-Lipschitz gradient %\big(that is $\BE\norm{\nabla H(\vw; \vxi_i)-\nabla H(\vw'; \vxi_i)}^2\leq \ell^2\norm{\vw-\vw'}^2$\big)
and convex, $h$ is $\mu$-strongly convex, and $\vxi_i$ is a random vector.
We have the following convergence result by using SARAH to solve problem (\ref{prob:convex-finite}).

\begin{thm}\label{thm:PSARAH}
For Algorithm \ref{alg:PiSARAH} with
\begin{align*}
K_0=\left\lceil\frac{\log\left(\|\fG_{\gamma}(\vw_{0})\|_2^2\right)}{\log2}\right\rceil, m=\left\lceil{256\kappa}\right\rceil \text{~~and~~} 
\lambda = \frac{1}{8\ell}. 
\end{align*}
\end{thm}
\begin{proof}
By following the proof of Corollary \ref{cor:SR} with $\tv_{k,0}=\nabla g(\tw_{k,0})$, we have
\begin{align*}
    \BE\|\fG_{\gamma}(\vw_{k+1})\|^2  
\leq & \frac{64\ell}{m\mu}\BE\|\fG_{\gamma}(\vw_{k})\|_2^2 + 8\ell^2\BE\norm{\tw_{k,1}-\tw_{k,0}}^2 \\
\leq & \left(\frac{64\ell}{m\mu}+\frac{1}{4}\right)\BE\|\fG_{\gamma}(\vw_{k})\|_2^2 \\
\leq & \frac{1}{2}\BE\|\fG_{\gamma}(\vw_{k})\|_2^2.
\end{align*}
Hence, we have $\BE\|\fG_{\gamma}(\vw_{K_0})\|_2^2 \leq \zeta$. 
\end{proof}

Similar to stochastic case, we directly obtain the following result.
\begin{cor}
    Under assumptions of Theorem \ref{thm:PSARAH}, we can obtain $\BE\|\fG_{\gamma}(\vw_{K_0})\|_2^2 \leq \kappa^{-2}\eps^{2}$ with $\fO\left((n+\kappa)\log(\kappa/\eps)\right)$ stochastic gradient evaluations.
\end{cor}
\begin{proof}
Using Theorem \ref{thm:PSARAH} with $\zeta=\kappa^{-2}\eps^{2}$, we have $\BE\|\fG_{\gamma}(\vw_{K_0})\|_2^2 \leq \kappa^{-2}\eps^{2}$.
The total number of stochastic gradient evaluation is
\begin{align*}
  & K_0 \cdot (n + m) \\ 
= & \left\lceil\frac{\log\left(\kappa^2\eps^{-2}\|\fG_{\gamma}(\vw_{0})\|_2^2\right)}{\log2}\right\rceil \cdot \left(n + \left\lceil{256\kappa}\right\rceil \right) \\
= & \fO\left((n+\kappa)\log(\kappa/\eps)\right).
\end{align*}
\end{proof}

\subsection{The case of $n\geq\kappa^2$}
We set the parameters
\begin{align*}
& \zeta = \kappa^{-2}\eps^2,~
\eta_k=\min\left(\dfrac{\eps}{5\kappa\ell\norm{\vv_k}}, \dfrac{1}{10\kappa\ell}\right),~
\lambda=\dfrac{1}{8\ell},~
q = \lceil\kappa^{-1}n^{1/2}\rceil, \\[0.1cm]
& S_2=\left\lceil\frac{3687}{76}\kappa q\right\rceil,~
K = \left\lceil\dfrac{100\kappa\ell \eps^{-2}\Delta_f}{9}\right\rceil,
~\text{and}~  m=\lceil1024\kappa\rceil.
\end{align*}
Then the quantity $\Delta_{k_0}$ is zero for any $k_0$ with ${\rm mod}~(k_0, q)=0$. We can follow all analysis of Theorem \ref{thm:main}. Note that the different values of $q$ and $\Delta_{k_0}$ do not affect the proof of Lemma \ref{lem:recursion}.
Therefore we still obtain $\BE\norm{\nabla \Phi(\hat\vx)}\leq 1504\eps$ by the parameters setting above. The total complexity of stochastic gradient evaluation is
\begin{align*}
  & \fO((n+\kappa)\log(\kappa/\eps))~+~\fO\left(\frac{K}{q}\cdot n\right)
    ~+~ \fO\left(K\cdot S_2\cdot m\right) \\
=& \fO((n+\kappa)\log(\kappa/\eps))~+~\fO\left(\frac{\kappa\eps^{-2}}{\kappa^{-1}n^{1/2}}\cdot n\right)  ~+~ \fO\left(\kappa\eps^{-2} \cdot n^{1/2}\cdot \kappa\right) \\
= & \fO\left(n\log(\kappa/\eps) + \kappa^2n^{1/2}\eps^{-2}\right).
\end{align*}

\subsection{The case of $n\leq\kappa^2$}
We set the parameters
\begin{align*}
    & \zeta = \kappa^{-2}\eps^2,~
\eta_k=\min\left(\dfrac{\eps}{5\kappa\ell\norm{\vv_k}}, \dfrac{1}{10\kappa\ell}\right),~
\lambda=\dfrac{2}{8\ell},~ q = 1, \\[0.1cm]
    & S_2 = 1, ~
    K = \left\lceil\dfrac{100\kappa\ell \eps^{-2}\Delta_f}{9}\right\rceil,~\text{and}~
    m = \left\lceil 1024\kappa  \right\rceil.
\end{align*}
The procedure of the algorithm means $\Delta_k=0$ holds for all $k$ since $q=1$.
Everything is identical to Theorem \ref{thm:main} until Lemma \ref{lem:recursion}. 
Then we revisit the derivation of Corollary \ref{cor:SR}.
Since we have $\Delta_k=0$, the inequalities (\ref{ieq:Delta-k}) and (\ref{ieq:delta-k}) will be tighter. Hence the all original bounds still hold. The remains could still follow the proof of Theorem \ref{thm:main} and we finally obtain $\BE\norm{\nabla \Phi(\hat\vx)}\leq 1504\eps$.

The total complexity of stochastic gradient evaluation is
\begin{align*}
  & \fO((n+\kappa)\log(\kappa/\eps))~+~\fO\left(\frac{K}{q}\cdot n\right)~+~ \fO\left(K\cdot S_2\cdot m\right) \\
=& \fO((n+\kappa)\log(\kappa/\eps))~+~\fO\left(\frac{\kappa\eps^{-2}}{1}\cdot n\right)
     ~+~ \fO\left(\kappa\eps^{-2} \cdot 1 \cdot \kappa\right) \\
=& \fO\left((\kappa^2 + \kappa n)\eps^{-2} \right).
\end{align*}

%  LocalWords:  Lipschitz polyak ieq tu SREDA nH StocGrad
%  LocalWords:  FullGradient ConcaveMaximizer

\end{document}